\documentclass[final]{siamltex}

\usepackage{amssymb,latexsym,amsmath}
\usepackage{epsfig}
\usepackage{caption}
\usepackage{caption}
 
 \usepackage{epsfig}

\usepackage{makeidx}

\allowdisplaybreaks

 \usepackage{graphicx,fancybox,latexsym,epsfig}
\usepackage{fancyhdr,amsmath,times,amsxtra,amssymb}
\usepackage{color}

\def\Pr{\mathop{\rm Pr}}

\def\B{{\mathcal B}}

\def\P{{\mathcal P}}

\def\sPr{{\mathsf{Pr}}}

\def\sX{{\mathds X}}

\def\sU{{\mathds U}}

\def\sU{{\mathds U}}

\usepackage{amsmath}
\usepackage{amsfonts}
\usepackage{array}
\usepackage{dsfont}
\usepackage{hyperref}
\usepackage{amssymb}
\usepackage{bbold}
\usepackage{caption}
\usepackage{standalone}
\usepackage{accents}
\usepackage{enumitem}
\usepackage{tikz}
\usepackage{pgfplots}
\pgfplotsset{compat=1.6}
\usepgfplotslibrary{groupplots}
\usepackage{cancel}
\allowdisplaybreaks

\newtheorem{assumption}{Assumption}[section]
  \newtheorem{remark}{Remark}[section]
  \newtheorem{example}{Example}[section]

\usepackage{amssymb,latexsym,amsmath}
\usepackage{mathrsfs}
\usepackage{epsfig}
\usepackage{caption}

\newcommand{\R}{\mathds{R}}
\newcommand{\Zplus}{\mathds{Z}_+}
\newcommand{\N}{\mathds{N}}

\newcommand{\dd}{\mathrm{d}}

\pgfplotsset{soldot/.style={color=blue,only marks,mark=*}}
\pgfplotsset{holdot/.style={color=blue,fill=white,only marks,mark=*}}
\fussy

\allowdisplaybreaks

\begin{document}

\sloppy
\title{Convergence of Finite Memory Q-Learning for POMDPs and Near Optimality of Learned Policies under Filter Stability
}
\author{Al\.{I} Devran Kara and Serdar Y\"uksel
\thanks{Al\.{I} Devran Kara is with the Department of Mathematics, University of Michigan, Ann Arbor, MI, USA, email: alikara@umich.edu. Serdar Y\"uksel is with the Department of Mathematics and Statistics,
     Queen's University, Kingston, ON, Canada,
     email: yuksel@queensu.ca}
     }
\maketitle
\begin{abstract}
In this paper, for POMDPs, we provide the convergence of a Q learning algorithm for control policies using a finite history of past observations and control actions, and, consequentially, we establish near optimality of such limit Q functions under explicit filter stability conditions. We present explicit error bounds relating the approximation error to the length of the finite history window. We establish the convergence of such Q-learning iterations under mild ergodicity assumptions on the state process during the exploration phase. We further show that the limit fixed point equation gives an optimal solution for an approximate belief-MDP. We then provide bounds on the performance of the policy obtained using the limit Q values compared to the performance of the optimal policy for the POMDP, where we also present explicit conditions using recent results on filter stability in controlled POMDPs. While there exist many experimental results, (i) the rigorous asymptotic convergence (to an approximate MDP value function) for such finite-memory Q-learning algorithms, and (ii) the near optimality with an explicit rate of convergence (in the memory size) under filter stability are results that are new to the literature, to our knowledge. 
 \end{abstract}


\section{Introduction}

Partially Observed Markov Decision Problems (POMDPs) offer a practically rich and relevant, and mathematically challenging model. Even in the most basic setup of finite state-action models, the analysis and computation of optimal solutions is complicated. The existence of optimal policies has in general been established via converting, or reducing, the original partially observed stochastic control problem to a fully observed Markov Decision Problem (MDP) with probability measure valued (belief) states, leading to a belief-MDP. However, computing an optimal policy for this fully observed model, and so for the original POMDP, using classical methods (such as dynamic programming, policy iteration, linear programming) is not simple even if the original system has finite state and action spaces, since the state space of the fully observed (reduced) model is always uncountable. Furthermore, when the dynamics are not known, learning theoretic methods have not been as comprehensively and conclusively studied as the fully observed counterpart for MDPs, mainly because of the the technical subtleties as we discuss further below.

{\bf On approximation methods.} The problem of approximate optimality is significantly more challenging compared to the fully observed counterpart. Most of the studies in the literature are algorithmic and computational contributions. These include \cite{porta2006point} and \cite{ZhHa01} which develop computational algorithms, utilizing structural convexity/concavity properties of the value function under the discounted cost criterion. \cite{spaan2005perseus} provides an insightful algorithm which may be regarded as a quantization of the belief space; however, no rigorous convergence results are provided. References \cite{smith2012point} and \cite{pineau2006anytime} also present quantization based algorithms for the belief state, where the state, measurement, and the action sets are finite. 

For partially observed setups, \cite{SYLTAC2017POMDP}, building on \cite{SaYuLi15c}, introduces a rigorous approximation analysis (and explicit methods for quantization of probability measures) after establishing weak continuity conditions on the transition kernel defining the (belief-MDP) via the non-linear filter \cite{FeKaZa12, KSYWeakFellerSysCont}, and shows that finite model approximations obtained through quantization are asymptotically optimal and the control policies obtained from the finite model can be applied to the actual system with asymptotically vanishing error as the number of quantization bins increases. Another rigorous set of studies is \cite{zhou2008density} and \cite{zhou2010solving} where the authors provide an explicit quantization method for the set of probability measures containing the belief states, where the state space is parametrically representable under strong density regularity conditions. The quantization is done through the approximations as measured by the Kullback-Leibler divergence (relative entropy) between probability density functions. \cite{Mahajan2019} presents a notion of approximate information variable and studies near optimality of policies that satisfies the approximate information state property. 

We refer the reader to the survey papers \cite{Lov91-(b),Whi91,hansen2013solving} and the recent book \cite{Kri16} for further structural results as well as algorithmic and computational methods for approximating POMDPs. Notably, for POMDPs \cite{Kri16} presents structural results on optimal policies under monotonicity conditions of the value function in the belief variable. 

{\bf On learning for POMDPs.} Learning in POMDPs is challenging for the reasons discussed above: if one attempts to learn optimal policies through empirical observations, then the analysis and convergence properties become significantly harder to obtain as the observations progress in a non-Markovian fashion and the belief state is uncountable. \cite{jaakkola1995reinforcement} studies a learning algorithm for POMDPs with average cost criteria where a policy improvement method is proposed using random polices and the convergence of this method to local optima is given. \cite{mccallum1997reinforcement} and \cite{lin1992memory} are studies that propose the same approach as we use in this paper, where they use a finite memory of history to construct learning algorithms. They provide extensive experimental results, however, both lack a rigorous convergence or approximation result.

A natural, though optimistic, suggestion to attempt to learn POMDPs would be to ignore the partial observability and pretend the noisy observations reflect the true state perfectly. For example, for infinite horizon discounted cost problems, one can construct Q iterations as:
\begin{align}\label{QPOMDP}
Q_{k+1}(y_k,u_k)=(1-\alpha_k(y_k,u_k))Q_k(y_{k},u_k)+\alpha_k(y_k,u_k)\left(C_k(y_k,u_k)+\beta \min_v Q_k(Y_{k+1},v)\right)
\end{align}
where $y_k$ represents the observations and $u_k$ represents the control actions. We can further improve this algorithm by using not only the most recent observation but a finite window of past observations and control actions since we can infer information on the true state from the past data. Two main problems with this approach are that (i) first, the $(Y_k,U_k)$ process is not a controlled Markov process (as only $(X_k,U_k)$ is) and the cost realizations $C_k(y_k,u_k)$ depend on the observation process in a random and a time-dependent fashion, and hence the convergence of this approach does not follow directly from usual techniques (\cite{jaakkola1994convergence, TsitsiklisQLearning}) and (ii) second, even if the convergence is guaranteed, it is not immediate what the limit Q values are, and whether they are meaningful at all. In particular, it is not known what MDP model gives rise to the limit Q values. 

\cite{singh1994learning} studied (\ref{QPOMDP}), that is the Q learning algorithm for POMDPs by ignoring the partial observability and constructing the algorithm using the most recent observation variable (where the state, action and measurements spaces were all assumed finite), and established convergence of this algorithm under mild conditions (notably that the hidden state process is uniquely ergodic under the exploration policy which is random and puts positive measure to all action variables). In our paper, we will consider memory sizes of more than $0$ for the information variables and a continuous state space, and thus the algorithm in \cite{singh1994learning} can be seen as a special case of our setup. Different from our work, however, \cite{singh1994learning} does not study what the limit of the iterations mean, and in particular whether the limit equation corresponds to some MDP model. In this paper, we rigorously construct the approximate belief MDP that the limit equation satisfies which gives an operational and practical conclusion regarding the analysis of the algorithm. Furthermore, we use different window sizes which turns out to be crucial for the performance of the learned policy: using longer window sizes reveals the intimate connection between the approximate learning problem and the nonlinear controlled filter stability problem that we will study in detail. This ultimately leads to near optimality of the $N$-window variation of (\ref{QPOMDP}) with an explicit approximation and robustness error bound as a function of $N$ and a computable/boundable coefficient related to filter stability.

Another motivation for our study is the following: often one deals with problems where not only the specification of an MDP is unknown, but {\it whether the problem is an MDP} in the first place may not be known. The simplest extension perhaps is that of a POMDP where one is tempted to view the measurements as the state, or finite window of measurement and actions as the state. A question, which has not been resolved fully, is whether a Q-learning algorithm for such a setup would indeed converge, and the next question is if it does, what it converges to. Our answer to the first question is positive under mild conditions; and the second question is, under filter stability conditions, that the convergence is to near optimality with an explicit error bound between the performance loss and the memory window size.

{\bf On finite-memory approximations and relations with controlled filter stability.} In our paper, we will see, perhaps not surprisingly, that filter stability is an essential ingredient for the learning algorithm to arrive at optimal or near optimal solutions. In other words, how fast the process forgets its initial prior distribution when updated with the information variables will be a key aspect for the performance of the approximate Q values determined using most recent information variables. Unlike fully observed systems, the system (belief-MDP) states cannot be visited infinitely often for POMDPs since there are uncountably many belief states and the measurements collected should somehow present approximate information on the belief states through conditions related to filter stability. We will make this intuition precise in our paper. We also note that in optimal control theory, it is a standard result that (time-invariant) output feedback control performs poorly compared with state-feedback and, in the absence of observability, this holds for all memory lengths.

We end the literature review section by mentioning particularly related studies on finite-memory control for POMDPs.
Reference \cite{white1994finite} is a particularly related work that studies approximation techniques for POMDPs using finite memory with finite state, action, measurements. The POMDP is reduced to a belief MDP and the worst and best case predictors prior to the $N$ most recent information variables are considered to build an approximate belief MDP. The original value function is bounded using these approximate belief MDPs that use only finite memory, where the finiteness of the state space is critically utilized. Furthermore, a loss bound is provided for a suboptimally constructed policy that only uses finite history, where the bound depends on a specific ergodicity coefficient (which requires restrictive sample path contraction properties). In this paper, we will consider more general signal spaces and consider more relaxed filter stability requirements, and, in particular, establish explicit rates of convergence results. We also rigorously construct the finite belief MDP considering the approximate Q learning algorithm whereas \cite{white1994finite} only focuses on the approximation aspect of POMDPs. 

In \cite{yu2008near}, the authors study near optimality of finite window policies for average cost problems where the state, action and observation spaces are finite; under the condition that the liminf and limsup of the average cost are equal and independent of the initial state, the paper establishes the near-optimality of (non-stationary) finite memory policies. Here, a concavity argument building on \cite{Feinberg2} (which becomes consequential by the equality assumption) and the finiteness of the state space is crucial. The paper shows that for any given $\epsilon>0$, there exists an $\epsilon$-optimal finite window policy. However, the authors do not provide a performance bound related to the length of the window, and in fact the proof method builds on convex analysis. 

In a recent paper \cite{kara2020near}, we established near optimality of finite window policies using a different approach by considering the belief-MDP directly and quantizing the belief space with a nearest neighbor map (under a metric on probability measures which induces the weak convergence topology) that uses finite window information variables. In particular, the results in that paper did not establish the convergence of a Q-learning algorithm and strictly speaking required the knowledge of the belief state to choose the nearest element from the finite set. As we will see later, the approximate Q learning algorithm does not necessarily choose the nearest element from the finite set induced by the window information variables. Thus, in this paper, we explicitly only use the memory variables directly for the approximation. We also note that the approximation method presented in \cite{kara2020near} only works for a restricted values of the discount factor which depends on the system components and the filter stability terms, whereas the method used in this paper does not put any restrictions on the discount factor. On the other hand, in  \cite{kara2020near} one could relax filter stability to be under weak convergence; in our current paper we consider filter stability under total variation. A detailed comparison is reported in Remark \ref{comparison}.

Similar with \cite{kara2020near}, our analysis here also makes explicit connections with filter stability; that is, how fast the controlled process forgets its initial distribution as it observes the information variables from the system. In the literature, there are various set of assumptions to achieve filter stability. Two main approaches have been:
\begin{itemize}
\item The transition kernel is sufficiently {\it ergodic}, forgetting the initial
measure and therefore passing this insensitivity (to incorrect initializations)
on to the filter process. This condition is often tailored towards control-free models.
\item The measurement channel provides sufficient information about the underlying state, allowing the filter to track the true state process. This approach is typically based on martingale methods and accordingly does not often lead to rates of convergence for the filter stability problem, but only asymptotic filter stability.
\end{itemize}
We use the recent results in the controlled filter stability literature presented in \cite{mcdonald2020exponential} for exponential filter stability and \cite{MYCDC2019observability,MYRobustControlledFS} for asymptotic filter stability.

In a recent study \cite{golowich2022planning}, a finite-memory based approximate planing method is studied for POMDPs, and the relation between the performance of the approximation and the filter stability is established similar to this paper and \cite{kara2020near}. To achieve filter stability, a restrictive rank condition is used for the observation channel, and a polynomial convergence rate is achieved as opposed to the general filter stability setup we consider here, which includes both exponential filter stability or asymptotic filter stability conditions. The approach in \cite{golowich2022planning} to deal with the filter stability is specifically tailored towards finite state spaces, whereas we present results for possibly continuous state spaces and our analysis in approximation is explicit for any filter stability error of the form given in $L_t$ (see equation (\ref{loss_constant})) and the setup in \cite{golowich2022planning} can be viewed as a particular instance where the state and action spaces are finite and the measurement channel has a restrictive invertibility condition. Furthermore, here we also present a reinforcement learning algorithm using finite-memory variables. We also emphasize that explicit filter stability conditions are provided in \cite[Theorem 3.3]{mcdonald2020exponential} for exponential filter stability and \cite[Theorem 3.6]{MYRobustControlledFS}  for asymptotic filter stability (the latter, via the examples in  \cite[Section 3]{mcdonald2018stability}, also includes a rank condition for finite models).

We highlight that one key contribution of the paper is the construction of an alternative belief MDP reduction introduced in Section \ref{alt_belief} which provides a structure to the finite-memory approximations. The alternative reduction technique leads to an explicit and rigorous error analysis by changing the topology and the construction of the state space of the reduced model with no restrictions on the discount parameter $\beta \in (0,1)$, unlike \cite{kara2020near}.

{\bf Contributions.} 
\begin{itemize}
\item[(i)] In Section \ref{alt_section}, we provide an alternative belief MDP reduction method that is tailored towards finite-memory approaches. In Section \ref{app_section}, we construct an approximate model using the alternative belief MDP reduction (see Figure \ref{FigFiniteWMDP}). In particular, in Theorem \ref{cont_bound} and Theorem \ref{robust_bound}, we establish bounds for the difference between the value functions of the original POMDP model and the approximate model, and  for the performance loss of the policy obtained using the approximate belief-MDP when it is used in the original model. We show that the policy obtained using the approximate model, uses finite-memory feedback variables to choose the control actions. Furthermore, Theorem \ref{cont_bound} and Theorem \ref{robust_bound} reveal the close connection between the finite-memory approximation method and the controlled filter stability problem through a filter stability term $L_t$ defined in (\ref{loss_constant}).
\item[(ii)] In Section \ref{q_section}, we present a Q-learning algorithm that uses finite-memory feedback variables. In Theorem \ref{main_thm}, we show that the Q iterations constructed using finite history variables converge under mild ergodicity assumptions on the hidden state process, and the limit fixed point equation corresponds to the optimal solution for the approximate belief-MDP model introduced in Section \ref{app_section}. 
\item[(iii)] We finally, in Section \ref{secfilterStable}, provide a particular result to guarantee exponential stability for controlled filter problems, which in turns implies that the error resulting from the finite-memory approximation and learning methods decays to $0$ exponentially fast as the memory size increases under explicit filter stability conditions to be presented (Corollary \ref{cor1} and Corollary \ref{cor2}).
\end{itemize}

In Section \ref{num_study}, we provide numerical examples which verify both the Q-learning convergence and near-optimality results.


\section{Partially Observed Markov Decision Processes and Belief-MDP Reduction}

Let $\mathds{X} \subset \mathds{R}^m$ denote a Borel set which is the state space of a partially observed controlled Markov process for some $m\in\mathds{N}$. Here and throughout the paper $\Zplus$ denotes the set of non-negative
integers and $\mathds{N}$ denotes the set of positive integers. Let
$\mathds{Y}$ be a finite set denoting the observation space of the model, and let the state be observed through an
observation channel $O$. The observation channel, $O$, is defined as a stochastic kernel (regular
conditional probability) from  $\mathds{X}$ to $\mathds{Y}$, such that
$O(\,\cdot\,|x)$ is a probability measure on the power set $P(\mathds{Y})$ of $\mathds{Y}$ for every $x
\in \mathds{X}$, and $O(A|\,\cdot\,): \mathds{X}\to [0,1]$ is a Borel
measurable function for every $A \in P(\mathds{Y})$.  A
decision maker (DM) is located at the output of the channel $O$, and hence it only sees the observations $\{Y_t,\, t\in \Zplus\}$ and chooses its actions from $\mathds{U}$, the action space which is also a finite set. An {\em admissible policy} $\gamma$ is a
sequence of control functions $\{\gamma_t,\, t\in \Zplus\}$ such
that $\gamma_t$ is measurable with respect to the $\sigma$-algebra
generated by the information variables
$
I_t=\{Y_{[0,t]},U_{[0,t-1]}\}, \quad t \in \mathds{N}, \quad
  \quad I_0=\{Y_0\},
$
where
\begin{equation}
\label{eq_control}
U_t=\gamma_t(I_t),\quad t\in \Zplus,
\end{equation}
are the $\mathds{U}$-valued control
actions and 
$Y_{[0,t]} = \{Y_s,\, 0 \leq s \leq t \}, \quad U_{[0,t-1]} =
  \{U_s, \, 0 \leq s \leq t-1 \}.$

\noindent We define $\Gamma$ to be the set of all such admissible policies. The update rules of the system are determined by (\ref{eq_control}) and the following
relationships:
\[  \Pr\bigl( (X_0,Y_0)\in B \bigr) =  \int_B \mu(dx_0)O(dy_0|x_0), \quad B\in \mathcal{B}(\mathds{X}\times\mathds{Y}), \]
where $\mu$ is the (prior) distribution of the initial state $X_0$, and
\begin{eqnarray*}
\label{eq_evol}
 \Pr\biggl( (X_t,Y_t)\in B \, \bigg|\, (X,Y,U)_{[0,t-1]}=(x,y,u)_{[0,t-1]} \biggr)
 = \int_B \mathcal{T}(dx_t|x_{t-1}, u_{t-1})O(dy_t|x_t),  
\end{eqnarray*}
$B\in \mathcal{B}(\mathds{X}\times\mathds{Y}), t\in \mathds{N},$ where $\mathcal{T}$ is the transition kernel of the model which is a stochastic kernel from $\mathds{X}\times
\mathds{U}$ to $\mathds{X}$. Note that, although $\mathds{Y}$ is finite, we here use integral sign instead of the summation sign for notation convenience by letting the measure to be sum of dirac-delta measures (and as we discuss later in the paper, our analysis will also hold for continuous measurement spaces).  We let the objective of the agent (decision maker) be the minimization of the infinite horizon discounted cost, 
  \begin{align}\label{criterion1}
    J_{\beta}(\mu,{\cal T},\gamma)= E_\mu^{{\cal T},\gamma}\left[\sum_{t=0}^{\infty} \beta^t c(X_t,U_t)\right]
  \end{align}
 \noindent for some discount factor $\beta \in (0,1)$, over the set of admissible policies $\gamma\in\Gamma$, where $c:\mathds{X}\times\mathds{U}\to\R$ is a Borel-measurable stage-wise cost function and $E_\mu^{{\cal T},\gamma}$ denotes the expectation with initial state probability measure $\mu$ and transition kernel ${\cal T}$ under policy $\gamma$. Note that $\mu\in\mathcal{P}(\mathds{X})$, where we let $\mathcal{P}(\mathds{X})$ denote the set of probability measures on $\mathds{X}$. We define the optimal cost for the discounted infinite horizon setup as a function of the priors and the transition kernels as
\begin{align*}
  J_{\beta}^*(\mu,{\cal T})&=\inf_{\gamma\in\Gamma} J_{\beta}(\mu,{\cal T},\gamma).
\end{align*}
For the analysis of partially observed MDPs, a common approach is to reformulate the problem as a fully observed MDP, where the decision maker keeps track of the posterior distribution of the state $X_t$ given the available history $I_t$. In the following section, we formalize this approach.

\subsection{Reduction to fully observed models using belief states}
\subsubsection{Convergence notions for probability measures}
For the analysis of the technical results, we will use different notions of convergence for sequences of probability measures.

Two important notions of convergence for sequences of probability measures are weak convergence, and convergence under total variation. For a complete, separable and metric space $\mathds{X}$, for a sequence $\{\mu_n,n\in\N\}$ in $\mathcal{P}(\mathds{X})$ is said to converge to $\mu\in\mathcal{P}(\mathds{X})$ \emph{weakly} if $\int_{\mathds{X}}c(x)\mu_n(dx) \to \int_{\mathds{X}}c(x)\mu(dx)$ for every continuous and bounded $c:\mathds{X} \to \R$.
One important property of weak convergence is that the space of probability measures on a complete, separable, metric (Polish) space endowed with the topology of weak convergence is itself complete, separable, and metric \cite{Par67}. One such metric is the bounded Lipschitz metric  (\cite[p.109]{villani2008optimal}), which is defined for $\mu,\nu \in \P(\mathds{X})$ as 
\begin{equation}\label{BLmetric}
\rho_{BL}(\mu,\nu):=\sup_{\|f\|_{BL}\leq1} | \int f d\mu - \int f d\nu | 
\end{equation}
where \[ \|f\|_{BL}:=\|f\|_\infty+\sup_{x\neq y}\frac{|f(x)-f(y)|}{d(x,y)} \]
and $\|f\|_\infty=\sup_{x\in\mathds{X}}|f(x)|$.

  For probability measures $\mu,\nu \in \mathcal{P}(\mathds{X})$, the \emph{total variation} metric is given by
  \begin{align*}
    \|\mu-\nu\|_{TV}&=2\sup_{B\in\mathcal{B}(\mathds{X})}|\mu(B)-\nu(B)|=\sup_{f:\|f\|_\infty \leq 1}\left|\int f(x)\mu(\dd x)-\int f(x)\nu(\dd x)\right|,
  \end{align*}
  \noindent where the supremum is taken over all measurable real $f$ such that $\|f\|_\infty=\sup_{x\in\mathds{X}}|f(x)|\leq 1$. A sequence $\mu_n$ is said to converge in total variation to $\mu \in \mathcal{P}(\mathds{X})$ if $\|\mu_n-\mu\|_{TV}\to 0$.

\subsubsection{Construction of the belief-MDP and some regularity properties}
It is by now a standard result that, for optimality analysis, any POMDP can be reduced to a completely observable Markov decision process \cite{Yus76}, \cite{Rhe74}, whose states are the posterior state distributions or {\it beliefs} of the observer or the filter process; that is, the state at time $t$ is
\begin{align}\label{belief_state}
z_t:=Pr\{X_{t} \in \,\cdot\, | Y_0,\ldots,Y_t, U_0, \ldots, U_{t-1}\} \in \P(\sX). 
\end{align}
We call this equivalent process the filter process \index{Belief-MDP}. The filter process has state space $\mathcal{Z} = \P(\sX)$ and action space $\sU$. Here, $\mathcal{Z}$ is equipped with the Borel $\sigma$-algebra generated by the topology of weak convergence \cite{Bil99}. Under this topology, $\mathcal{Z}$ is a standard Borel space \cite{Par67}. 
Then, the transition probability $\eta$ of the filter process can be constructed as follows (see also \cite{Her89}). If we define the measurable function 
\[F(z,u,y) := F(\,\cdot\,|z,u,y) = Pr\{X_{t+1} \in \,\cdot\, | Z_t = z, U_t = u, Y_{t+1} = y\}\]
 from ${\cal P}(\mathds{X})\times\mathds{U}\times\mathds{Y}$ to ${\cal P}(\mathds{X})$ and use the stochastic kernel $P(\,\cdot\, | z,u) = \Pr\{Y_{t+1} \in \,\cdot\, | Z_t = z, U_t = u\}$ from ${\cal P}(\mathds{X})\times\mathds{U}$ to $\mathds{Y}$, we can write $\eta$ as
\begin{align}
\eta(\,\cdot\,|z,u) = \int_{\mathds{Y}} 1_{\{F(z,u,y) \in \,\cdot\,\}} P(dy|z,u). \label{beliefK}
\end{align}

The one-stage cost function $\tilde{c}:{\cal P}(\mathds{X}) \times \mathds{U}\rightarrow[0,\infty)$ of the filter process is given by 
\begin{align}\label{belief_cost}
\tilde{c}(z,u) := \int_{\sX} c(x,u) z(dx),
\end{align}
which is a Borel measurable function. Hence, the filter process is a completely observable Markov process with the components $(\mathcal{Z},\sU,\tilde{c},\eta)$.

For the filter process, the information variables is defined as
\[
\tilde{I}_t=\{Z_{[0,t]},U_{[0,t-1]}\}, \quad t \in \mathds{N}, \quad
  \quad \tilde{I}_0=\{Z_0\}.
\]

It is well known that an optimal control policy of the original POMDP can use the belief $Z_t$ as a sufficient statistic for optimal policies (see \cite{Yus76}, \cite{Rhe74}), provided they exist. More precisely, the filter process is equivalent to the original POMDP  in the sense that for any optimal policy for the filter process, one can construct a policy for the original POMDP which is optimal. On existence, we note the following. 

By the recent results in \cite{FeKaZg14} and \cite{KSYWeakFellerSysCont} the transition model of the belief-MDP can be shown to satisfy weak continuity conditions on the belief state and action variables, and accordingly we have that the measurable selection conditions \cite[Chapter 3]{HernandezLermaMCP} apply.  Notably, we state the following.

\begin{assumption}\label{TV_channel}
\begin{itemize}
\item[(i)] The transition probability $\mathcal{T}(\cdot|x,u)$ is weakly continuous in $(x,u)$, i.e., for any $(x_n,u_n)\to (x,u)$, $\mathcal{T}(\cdot|x_n,u_n)\to \mathcal{T}(\cdot|x,u)$ weakly.
\item[(ii)] The observation channel $O(\cdot|x)$ is continuous in total variation, i.e., for any $x_n\to x$, $O(\cdot|x_n) \rightarrow O(\cdot|x)$ in total variation.
\end{itemize}
\end{assumption}

\begin{assumption}\label{TV_kernel}
The transition probability $\mathcal{T}(\cdot|x,u)$ is continuous in total variation in $(x,u)$, i.e., for any $(x_n,u_n)\to (x,u)$, $\mathcal{T}(\cdot|x_n,u_n) \to \mathcal{T}(\cdot|x,u)$ in total variation.
\end{assumption}

\begin{theorem} 
\begin{itemize}
\item[(i)] \cite{FeKaZg14} \label{TV_channel_thm}
Under Assumption \ref{TV_channel}, the transition probability $\eta(\cdot|z,u)$ of the filter process is weakly continuous in $(z,u)$.
\item[(ii)] \cite{KSYWeakFellerSysCont} \label{TV_kernel_thm}
Under Assumption \ref{TV_kernel}, the transition probability $\eta(\cdot|z,u)$ of the filter process is weakly continuous in $(z,u)$.
\end{itemize}
\end{theorem}

Under the above weak continuity conditions and appropriate conditions on the stage-wise cost function (e.g. bounded and continuous $c$ with Assumption \ref{TV_channel_thm} or bounded $c$ with Assumption \ref{TV_kernel_thm}), the measurable selection conditions \cite[Chapter 3]{HernandezLermaMCP} apply and a solution to the discounted cost optimality equation exists, and accordingly an optimal control policy exists.

 This policy is stationary (in the belief state). If we denote this optimal belief policy by $\phi : \P(\mathds{X}) \to \mathds{U}$, we can then find a policy $\gamma$ on the partially observed setup such that
\begin{align*}
\gamma(y_{[0,n]}):=\phi\left(P^{\mu,\gamma}(X_n\in\cdot|Y_{[0,n]}=y_{[0,n]})\right)=\phi(\pi_n^{\mu,\gamma}).
\end{align*}
Hence, the policy $\gamma$ can be used as an optimal policy for the partially observed MDP.


Even though, the belief MDP approach provides a strong tool for the analysis of POMDPs, it is usually too complicated computationally. The belief space $\mathcal{Z}=\P(\mathds{X})$ is always uncountable even when $\mathds{X}$, $\mathds{Y}$
 and  $\mathds{U}$ are finite. Furthermore, the the information variables $I_t$ grows with time and the computation of the belief state $Pr(X_t \in \cdot | I_t)$ can become intractable. Therefore, approximation of the belief-MDP is usually needed. In the following section, we provide an alternative fully observed MDP approach and present approximation results that only make use of a finite history of the information variables.

\section{An Alternative Finite Window Belief-MDP Reduction and its Approximation}\label{alt_belief}

\subsection{An alternative finite window belief-MDP reduction}\label{alt_section}

In this section we construct an alternative fully observed MDP reduction with the condition that the controller has observed at least $N$ information variables, using the predictor from $N$ stages earlier and the most recent $N$ information variables (that is, measurements and actions). This new construction allows us to highlight the most recent information variables and {\it compress} the information coming from the past history via the predictor as a probability measure valued variable. In what follows, we will sometimes consider the case with $N=1$ for some of the proofs to make the presentation less complicated. The general case follows from identical arguments.

For the remainder of the paper, to emphasize the prior distribution of the starting state variable, we will use the following notation for conditional probabilities on state and observation variables.
\begin{definition}
Assume that the initial state $X_0$ has a prior distribution $\mu\in\P(\mathds{X})$. Then, for the conditional distribution of $X_t$ given the past observation and action variables $\{y_t,\dots,y_0\}$, $\{u_{t-1},\dots,u_0\}$ we define
\begin{align*}
P^\mu(X_t\in\cdot|y_t,\dots,y_0,u_{t-1},\dots,u_0):=Pr(X_t\in\cdot|y_t,\dots,y_0,u_{t-1},\dots,u_0).
\end{align*}
Given that $x_0$ has a prior distribution $\mu\in\P(\mathds{X})$, we define the following for the conditional distribution of $Y_t$ given the past observation and action variables $\{y_{t-1},\dots,y_0\}$, $\{u_{t-1},\dots,u_0\}$ 
\begin{align*}
P^\mu(Y_t\in\cdot|y_{t-1},\dots,y_0,u_{t-1},\dots,u_0):=Pr(Y_t\in\cdot|y_{t-1},\dots,y_0,u_{t-1},\dots,u_0).
\end{align*}
\end{definition}

Consider the following state variable at time $t$:
\begin{align}\label{finite_belief_state}
\hat{z}_t=(\pi_{t-N}^-,I_t^N)
\end{align}
where, for $N\geq 1$
\begin{align*}
\pi_{t-N}^-&=Pr(X_{t-N}\in \cdot|y_{t-N-1},\dots,y_0,u_{t-N-1},\dots,u_0),\\
I_t^N&=\{y_t,\dots,y_{t-N},u_{t-1},\dots,u_{t-N}\}
\end{align*}
and $I_t^N=y_t$ for $N=0$
with $\mu$ being the prior probability measure on $X_0$. The state space with this representation is $\hat{\mathcal{Z}}=\P(\mathds{X})\times \mathds{Y}^{N+1}\times \mathds{U}^{N}$ where we equip $\hat{\mathcal{Z}}$ with the product topology where we consider the weak convergence topology on the $\P(\mathds{X})$ coordinate and the usual (coordinate) topologies on $\mathds{Y}^{N+1}\times \mathds{U}^{N}$ coordinates.

This new state representation can be mapped to the belief state $z_t$ defined in (\ref{belief_state}). Consider the map $\psi: \hat{\mathcal{Z}}\to \P(\mathds{X})$, for some $\hat{z}_t=(\pi_{t-N}^-,I_t^N)$
\begin{align*}
\psi(\hat{z}_t)=\psi(\pi_{t-N}^-,I_t^N)&=P^{\pi_{t-N}^-}(X_t\in\cdot|I_t^N)=P^{\pi_{t-N}^-}(X_t\in\cdot|y_t,\dots,y_{t-N},u_{t-N-1},\dots,u_{t-N-1})\\
&=P^\mu(X_t\in\cdot|y_t,\dots,y_0,u_{t-1},\dots,u_0)=z_t
\end{align*}
such that the map $\psi$ acts as a Bayesian update of $\pi_{t-N}^-$ using $I_t^N$. Using this map, we can define the stage-wise cost function and the transition probabilities. Consider the new cost function $\hat{c}:\hat{\mathcal{Z}}\times \mathds{U}\to \mathds{R}$, using the cost function $\tilde{c}$ of the belief MDP (defined in (\ref{belief_cost})) such that 
\begin{align}\label{hat_cost}
\hat{c}(\hat{z}_t,u_t)&=\hat{c}(\pi_{t-N}^-,I_t^N,u_t)=\tilde{c}(\psi(\pi_{t-N}^-,I_t^N),u_t)\nonumber\\
&=\int_\mathds{X}c(x_t,u_t)P^{\pi^-_{t-N}}(dx_t|y_t,\dots,y_{t-N},u_{t-1},\dots,u_{t-N}).
\end{align}
Furthermore, we can define the transition probabilities as follows: for some $A\in \B(\hat{Z})$ such that 
\[\ A= B\times\{\hat{y}_{t-N+1},\hat{u}_{t},\dots,\hat{u}_{t-N+1}\},\quad  B\in\B(\P(\mathds{X})) \]
we write
\begin{align*}
&Pr(\hat{z}_{t+1}\in A|\hat{z}_{t},\dots,\hat{z}_0,u_{t},\dots,u_0) \\
&=Pr(\pi_{t-N+1}^-\in B,\hat{y}_{t+1},\dots,\hat{y}_{t-N+1},\hat{u}_{t},\dots,\hat{u}_{t-N+1}|\pi_{t-N}^-,\dots,\pi_0^-,y_t,\dots,y_0,u_t,\dots,u_0)\\
&=\mathds{1}_{\{\left(y_{t},\dots,y_{t-N+1},u_{t},\dots,u_{t-N+1}\right)=\left(\hat{y}_{t},\dots,\hat{y}_{t-N+1},\hat{u}_{t},\dots,\hat{u}_{t-N+1}\right)\}}\\
&\qquad\qquad\times \mathds{1}_{\{G(\pi_{t-N}^-,y_{t-N},u_{t-N})\in B\}}P^{\pi_{t-N}^-}(\hat{y}_{t+1}|y_t,\dots,y_{t-N},u_{t},\dots,u_{t-N})\\
&=Pr(\pi_{t-N+1}^-\in B,\hat{y}_{t+1},\dots,\hat{y}_{t-N+1},\hat{u}_{t},\dots,\hat{u}_{t-N+1}|\pi_{t-N}^-,y_t,\dots,y_{t-N},u_t,\dots,u_{t-N})\\
&=Pr(\hat{z}_{t+1}\in A|\hat{z}_t,u_t)\\
&=:\int_A\hat{\eta}(d\hat{z}_{t+1}|\hat{z}_t,u_t)
\end{align*}
where the map $G$ is defined as 
\begin{align*}
&G(\pi_{t-N}^-,y_{t-N},u_{t-N})=G(P^{\mu}(X_{t-N}\in\cdot|y_{t-N-1},\dots,y_0,u_{t-N-1},\dots,u_0),y_{t-N},u_{t-N})\\
&=P^\mu(X_{t-N+1}\in\cdot|y_{t-N},\dots,y_0,u_{t-N},\dots,u_0).
\end{align*}

Hence, $\hat{\eta}$ defines a controlled transition model for the new states $\hat{z}_{t+1}\in \hat{\mathcal{Z}}$. Then, we have a proper fully observed MDP, with the cost function $\hat{c}$, transition kernel $\hat{\eta}$ and the state space $\hat{\mathcal{Z}}$.

Note that any policy $\phi:\P(\mathds{X})\to \mathds{U}$ defined for the belief MDP, can be extended to the newly defined finite window belief-MDP using the map $\psi$, and defining $\hat{\phi}:=\phi\circ\psi$ such that 
\begin{align*}
\hat{\phi}(\hat{z})=\phi(\psi(z)).
\end{align*}
Thus, if an optimal policy can be found for the belief MDP, say $\phi^*$, the policy $\hat{\phi}^*=\phi^*\circ\psi$ is an optimal policy for the newly defined MDP. 

We now write the discounted cost optimality equation for the newly constructed finite window belief MDP. Note that with the alternative approach the state $\hat{z}$ can only be written, if we have at least $N$ information variables. Therefore, given that the decision maker observed at least $N$ information variables, we write the following fixed point equation
\begin{align*}
J^*_\beta(\hat{z})&=\min_{u\in\mathds{U}}\left(\hat{c}(\hat{z},u)+\beta \int J_\beta^*(\hat{z}_1)\hat{\eta}(d\hat{z}_1|\hat{z},u)\right).
\end{align*}
We can rewrite this fixed point equation in a different form, for notation ease assume $N=1$.
If $\hat{z}$ has the form $(\pi_0^-,y_1,y_0,u_0)$, then we can rewrite 
\begin{align}\label{fixed_wind}
&J_\beta^*(\pi_0^-,y_1,y_0,u_0)\nonumber\\
&=\min_{u_{1}\in\mathds{U}}\bigg(\hat{c}(\pi_0^-,y_1,y_0,u_0,u_1)+\beta \sum_{y_2\in\mathds{Y}} J_\beta^*(\pi_1^-(\pi_0^-,y_0,u_0),y_{2},y_1,u_{1})P^{\pi_0^-}(y_{2}|y_1,y_0,u_1,u_0)\bigg).
\end{align}
This representation will play an important role in the analysis of the problem. Note that the policy $\hat{\phi}^*=\phi^*\circ \psi$ satisfies this fixed point equation.

The following fixed point equation can also be defined for any policy $\hat{\phi}:\hat{\mathcal{Z}}\to \mathds{U}$
\begin{align*}
J_\beta(\hat{z},\hat{\phi})=\hat{c}(\hat{z},\hat{\phi}(\hat{z}))+\beta \int J_\beta(\hat{z}_1,\hat{\phi})\hat{\eta}(d\hat{z}_1|\hat{z},\hat{\phi}(\hat{z}))
\end{align*}
where $J_\beta(\hat{z},\hat{\phi})$ denotes the value function under the policy $\hat{\phi}$ for the initial point $\hat{z}$.

\subsection{Approximation of the finite window belief-MDP}\label{app_section}
We now approximate the MDP constructed in the previous section. Consider the following set $\hat{\mathcal{Z}}_{\pi^*}^N$ for a fixed $\pi^*\in\P(\mathds{X})$
\begin{align}\label{finite_set_belief}
\hat{\mathcal{Z}}_{\pi^*}^N=\bigg\{(\pi^*,y_{[0,N]},u_{[0,N-1]}): y_{[0,N]}\in\mathds{Y}^{N+1}, u_{[0,N-1]}\in\mathds{U}^N\bigg\}
\end{align}
such that the state at time $t$ is $\hat{z}_t^N=(\pi^*,I_t^N)$. Compared to the state $\hat{z}_t=(\pi_{t-N}^-,I_t^N)$ defined in (\ref{finite_belief_state}), this approximate model uses $\pi^*$ as the predictor, no matter what the real predictor at time $t-N$ is.

The cost function is defined in usual manner so that 
\begin{align*}
\hat{c}(\hat{z}^N_t,u_t)&=\hat{c}(\pi^*,I_t^N,u_t)=\tilde{c}(\phi(\pi^*,I_t^N),u_t)\\
&=\int_\mathds{X}c(x_t,u_t)P^{\pi^*}(dx_t|y_t,\dots,y_{t-N},u_{t-1},\dots,u_{t-N}).
\end{align*}

We define the controlled transition model by 
\begin{align}\label{eta_N}
\hat{\eta}^N(\hat{z}_{t+1}^N|\hat{z}_t^N,u_t)=\hat{\eta}^N(\pi^*,I_{t+1}^N|\pi^*,I_t^N,u_t):=\hat{\eta}\bigg(\P(\mathds{X}),I_{t+1}^N|\pi^*,I_t^N,u_t\bigg).
\end{align}

For simplicity, if we assume $N=1$, then the transitions can be rewritten for some $I_{t+1}^N=(\hat{y}_{t+1},\hat{y}_t,\hat{u}_t)$ and $I_t^N=(y_t,y_{t-1},u_{t-1})$
\begin{align}\label{eta_N_1}
\hat{\eta}^N(\pi^*,\hat{y}_{t+1},\hat{y}_t,\hat{u}_t|\pi^*,y_t,y_{t-1},u_{t-1},u_t)&=\hat{\eta}(\P(\mathds{X}),\hat{y}_{t+1},\hat{y}_t,\hat{u}_t|\pi^*,y_t,y_{t-1},u_{t-1},u_t)\nonumber\\
&=\mathds{1}_{\{y_t=\hat{y}_t,u_t=\hat{u}_t\}}P^{\pi^*}(\hat{y}_{t+1}|y_t,y_{t-1},u_t,u_{t-1}).
\end{align} 

Denoting the optimal value function for the approximate model by $J_\beta^N$, we can write the following fixed point equation
\begin{align}\label{N_fixed}
J_\beta^N(\hat{z}^N)=\min_{u\in\mathds{U}}\left(\hat{c}(\hat{z}^N,u)+\beta\sum_{\hat{z}_1^N\in\hat{\mathcal{Z}}^N_{\pi^*}}J_\beta^N(\hat{z}_1^N)\hat{\eta}^N(\hat{z}_1^N|\hat{z}^N,u)\right).
\end{align}
By assuming $N=1$ again, we can rewrite the fixed point equation for some $\hat{z}_0^N=(\pi^*,y_1,y_0,u_0)$ as
\begin{align}\label{N_1_fixed}
J_\beta^N(\pi^*,y_1,y_0,u_0)=\min_{u_1\in\mathds{U}}\left(\hat{c}(\pi^*,y_1,y_0,u_0,u_1)+\beta\sum_{y_2\in\mathds{Y}}J_\beta^N(\pi^*,y_2,y_1,u_1)P^{\pi^*}(y_2|y_1,y_0,u_1,u_0)\right).
\end{align}

Since everything is finite in this setup, we can assume the existence of an optimal policy $\phi^N$ that satisfies this fixed point equation. Note that both $J^N_\beta$ and $\phi^N$ are defined on the finite set $\hat{\mathcal{Z}}^N_{\pi^*}$. However, we can simply extend them to the set $\hat{\mathcal{Z}}$ by defining
\begin{align*}
\tilde{J}^N_\beta(\hat{z})=\tilde{J}^N_\beta(\pi,y_1,y_0,u_0)&:=J^N_\beta(\pi^*,y_1,y_0,u_0)\\
\tilde{\phi}^N(\hat{z})=\tilde{\phi}^N(\pi,y_1,y_0,u_0)&:=\phi^N(\pi^*,y_1,y_0,u_0)
\end{align*}
for any $\hat{z}=(\pi,y_1,y_0,u_0)\in\hat{\mathcal{Z}}$.

We will later prove that Q-value iterations using finite window of information variables converge to the Q-values for the approximate model constructed in this section. For $N=1$, for example, equation (\ref{N_1_fixed}), will be significant for the Q-value iteration.

Another point to note is that the policy $\phi^N$ only uses most recent $N$ information variables to choose the control actions.

In what follows, we investigate the following differences
\begin{align*}
&|\tilde{J}^N_\beta(\hat{z})-J^*_\beta(\hat{z})|,\\
&J_\beta(\hat{z},\tilde{\phi}^N) -J^*_\beta(\hat{z}).
\end{align*}
The first one is the difference between the optimal value function of the original model and that for the approximate model. The second term is the performance loss due to the policy calculated for the approximate model being applied to the true model.

\begin{remark}
 We note that, in \cite{SaYuLi15c}, the authors study approximation methods for MDPs with continuous state spaces by quantizing the state space and constructing a finite state MDP. In this section, we also construct a finite state space, $\hat{\mathcal{Z}}_{\pi^*}^N$, by quantizing $\hat{\mathcal{Z}}$. In  \cite{SaYuLi15c}, continuity properties of the transition kernel, $\hat{\eta}$ are used. However, establishing regularity properties for $\hat{\eta}$ is challenging. Therefore, we follow a different approach and instead of working directly with $\hat{\eta}$, we analyze the components of partially observed MDP, for the following approximation results. We note that, our quantization method is tailored towards filter stability, and corresponds to a uniform quantization when we endow the finite window belief MDP space $\hat{\mathcal{Z}}=\P(\mathds{X})\times \mathds{Y}^{N+1}\times \mathds{U}^{N}$ with the product topology of the weak convergence topology on $\P(\mathds{X})$ and the usual (coordinate) topologies on $\mathds{Y}$ and $\mathds{U}$. We also note that our approach here then naturally applies to continuous (such as finite dimensional real valued) but compact space valued measurement and action spaces as well, as a uniform quantization can be applied for all finite window belief MDP realizations. \hfill $\diamond$
\end{remark}

\subsubsection{Difference in the value functions in terms of a uniform filter stability error}
In this section, we study the difference $|\tilde{J}^N_\beta(\hat{z})-J^*_\beta(\hat{z})|$.

Before the result, we introduce some notation. We first define the measurable policies with respect to the new state space $\hat{\mathcal{Z}}=\P(\mathds{X})\times \mathds{Y}^{N+1}\times \mathds{U}^{N}$ by $\hat{\Gamma}$. That is, a policy $\hat{\gamma}\in\hat{\Gamma}$ is a sequence of control functions $\{\hat{\gamma}_t,\, t\in \Zplus\}$ such
that $\hat{\gamma}_t$ is measurable with respect to the $\sigma$-algebra
generated by the information variables $\{\hat{z}_0,\dots,\hat{z}_t\}$.

We now define the following bounding term
\begin{align}\label{loss_constant}
L^N_t:=\sup_{\hat{\gamma}\in\hat{\Gamma}}E_{\pi_0^-}^{\hat{\gamma}}\left[\|P^{\pi_t^-}(X_{t+N}\in\cdot|Y_{[t,t+N]},U_{[t,t+N-1]})-P^{\pi^*}(X_{t+N}\in\cdot|Y_{[t,t+N]},U_{[t,t+N-1]})\|_{TV}\right]
\end{align} 
which is the expected bound on the total variation distance between the posterior distributions of $X_{t+N}$ conditioned on the same observation and control action variables $Y_{[t,t+N]},U_{[t,t+N-1]}$ when the prior distributions of $X_{t}$ are given by $\pi_t^-$ and $\pi^*$. The expectation is with respect to the random realizations of $\pi_t^-$ and $Y_{[t,t+N]},U_{[t,t+N-1]}$ under the true dynamics of the system when the prior distribution of $x_0$ is given by $\pi_0^-$.  This constant represents the bound on the distance of two processes with different starting points when they are updated with the same observation and action processes under same policy. This is related to the filter stability problem, which will be discussed in Section \ref{secfilterStable}.

For the remaining of the paper, we will drop the $N$ dependence and denote the term by $L_t$.

\begin{theorem}\label{cont_bound}
For $\hat{z}_0=(\pi_0^-,I_0^N)$, if a policy $\hat{\gamma}$ acts on the first $N$ step of the process which produces $I_0^N$, we then have
\begin{align*}
E_{\pi_0^-}^{\hat{\gamma}}\left[\left|\tilde{J}^N_\beta(\hat{z}_0)-J^*_\beta(\hat{z}_0)\right||I_0^N\right]\leq  \frac{\|c\|_\infty }{(1-\beta)}\sum_{t=0}^\infty\beta^tL_t
\end{align*}
where $L_t$ is defined as in ($\ref{loss_constant}$).
\end{theorem}
\begin{proof}{Proof.}
The proof can be found in Appendix \ref{proof_cont_bound}.
\end{proof}

\subsubsection{Performance loss due to approximate policy being applied to the true system in terms of the filter stability error} We now study the difference $J_\beta(\hat{z},\tilde{\phi}^N) -J^*_\beta(\hat{z})$ where $\tilde{\phi}^N$ is the optimal policy for the approximate model extended to the full space $\hat{\mathcal{Z}}$.

\begin{theorem}\label{robust_bound}
For $\hat{z}_0=(\pi_0^-,I_0^N)$, with a policy $\hat{\gamma}$ acting on the first $N$ steps
\begin{align*}
E_{\pi_0^-}^{\hat{\gamma}}\left[\left|J_\beta(\hat{z}_0,\tilde{\phi}^N) -J^*_\beta(\hat{z}_0)\right||I_0^N\right]\leq  \frac{2\|c\|_\infty }{(1-\beta)}\sum_{t=0}^\infty\beta^tL_t.
\end{align*}
\end{theorem}
\begin{proof}{Proof.}
The proof can be found in Appendix \ref{proof_robust_bound}.
\end{proof}

\begin{remark}\label{comparison}
In \cite{kara2020near}, we constructed finite state approximate belief MDP using the state space $\hat{Z}_{\pi^*}^N$ defined in (\ref{finite_set_belief}). However, different from the approach we use in this paper, to determine the approximate states, we used a nearest neighbor map, to choose the closest element from the set $\hat{Z}_{\pi^*}^N$ to the $\P(\sX)$-valued belief state $z_t:=\sPr\{X_{t} \in \,\cdot\, | Y_0,\ldots,Y_t, U_0, \ldots, U_{t-1}\}$ under the {\it bounded Lipchitz} (BL) metric. We recall that the bounded Lipchitz metric, $\rho_{BL}$, for some $\mu,\nu\in\P(\mathds{X})$ is given in (\ref{BLmetric}).
%
%
To find the closest element from $\hat{Z}_{\pi^*}^N$ one needs to know the belief state realization $z_t$ and to calculate/update the belief state, the system dynamics need to be known. However, as we will see later, the $Q$ learning algorithm presented here, using only the finite window information variables $I_t^N$,  converges to the optimality equation of an approximate belief MDP that maps the belief state to an element from $\hat{Z}_{\pi^*}^N$ with matching finite window information rather than the closest element under the bounded Lipschitz metric. Hence, the alternative belief MDP construction and the approximation setup we have presented in this section serves better to analyze the approximate Q-learning algorithm which strictly uses the finite-window memory variables. In other words, one does not need to calculate the belief state but only needs to keep track of the information variables $I_t^N$ for the approximation method introduced in this section. In particular, the state $(\pi_{t-N}^-,I_t^N)$ is always mapped/quantized to $(\pi^*,I_t^N)$ which can be done without the knowledge or computation of $\pi_{t-N}^-$ as long as we have $I_t^N$ available. Furthermore, In \cite{kara2020near}, since we directly work with the topology and the metrics on the space of probability measures, the distinction between different realizations of history variables might be lost, as we only care about the resulting posterior distribution on the hidden state variable, e.g., different realizations of history variables may produce same posterior distribution. However, for the finite-memory Q-learning iterations, it is key to be able to differentiate between the different realizations of finite-memory feedback variables. Hence, in this paper, we put a different topology on $\hat{Z}_{\pi^*}^N$, by separating the finite-memory variables, rather than directly working with probability measures. This approach helps us to distinguish between different finite-memory realizations.

On the other hand, one advantage of the approximation scheme used in \cite{kara2020near} is that because of the nearest neighbourhood map, one naturally arrives at a smaller approximation error. Furthermore, because of the continuity properties of the nearest neighbor map under the BL metric, one is able to work with the weak convergence topology, as such, we get an upper bound in terms of the BL metric $\rho_{BL}$, such that the bounding term is \[\rho_{BL}\left(P^{\pi}(X_t\in\cdot|y_{[t,t-N]},u_{[t-1,t-N]}),P^{\pi^*}(X_t\in\cdot|y_{[t,t-N]},u_{[t-1,t-N]})\right)\] which is always dominated by the total variation metric that we use in this section. \hfill $\diamond$
\end{remark}

The different formulations and the approximation approach are summarized in Figure \ref{FigFiniteWMDP}.

\begin{figure}[h]
\centering
\epsfig{figure=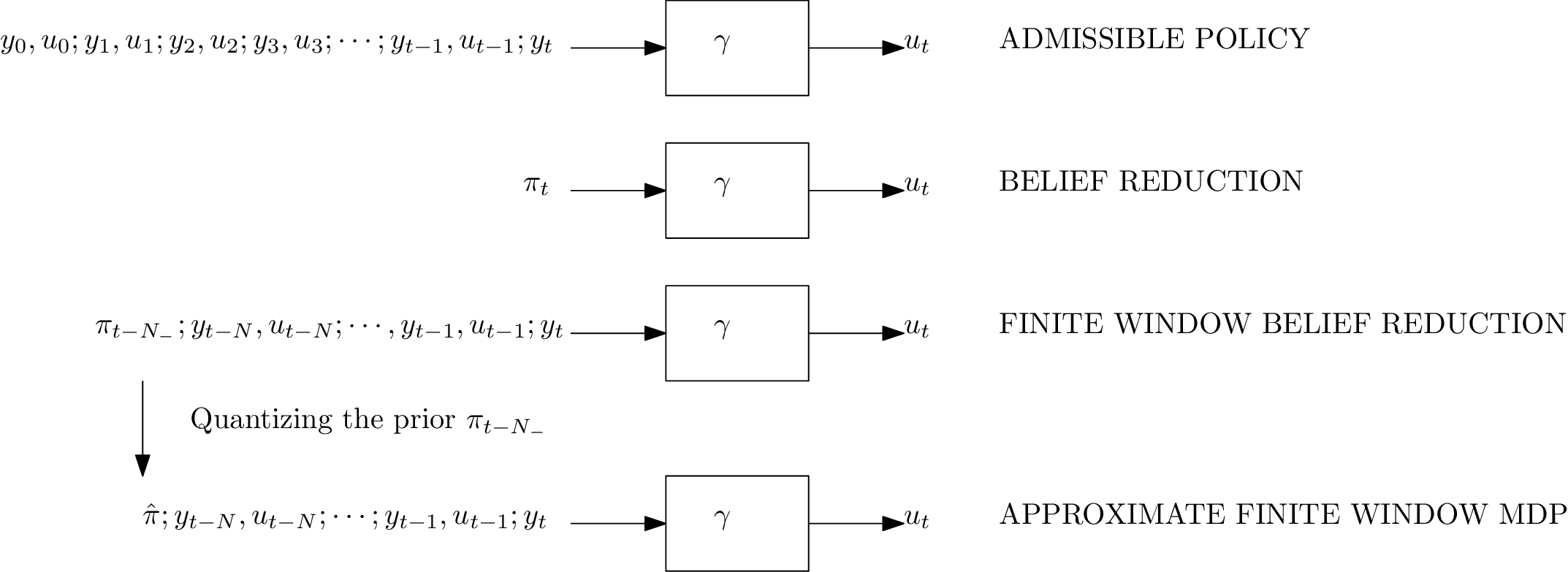,height=6cm,width=15cm}
\caption{Construction of the Finite-Window Approximate MDP from the Finite-Window Belief-MDP. The quantization of the finite window MDP model leads to the collapse of the first coordinate to a fixed measure.}\label{FigFiniteWMDP}
\end{figure}

\section{Q Iterations Using a Finite History of Information Variables and Convergence}\label{q_section}


Assume that we start keeping track of the last $N+1$ observations and the last $N$ control action variables after at least $N+1$ time steps. That is, at time $t$, we keep track of the information variables
\begin{align*}
I^N_t=\begin{cases}&\{y_t, y_{t-1},\dots,y_{t-N},u_{t-1},\dots,u_{t-N}\} \quad \text{ if } N>0\\
&y_t  \quad \text{ if } N=0.\end{cases}
\end{align*}
We will construct the Q-value iteration using these information variables. In what follows, we will drop the $N$ dependence on $I_t^N$ and sometimes we will use $N=1$ for simplicity of the notation. For these new approximate states, we follow the usual Q learning algorithm such that for any $I\in\mathds{Y}^{N+1}\times\mathds{U}^N$ and $u\in\mathds{U}$
\begin{align}\label{q_alg}
Q_{t+1}(I,u)=(1-\alpha_t(I,u))Q_t(I,u)+\alpha_t(I,u)\left(C_t(I,u)+\beta \min_v Q_t(I^t_1,v)\right),
\end{align}
where $I^t_1 = \{Y_{t+1}, y_{t},\dots,y_{t-N+1},u_{t},\dots,u_{t-N+1}\}$, we put the $t$ dependence to emphasize that the distribution of $Y_{t+1}$ and hence $I_1^t$ are different for every $t$. 


To choose the control actions, we use polices that choose the control actions randomly and independent of everything else such that at time $t$
\begin{align*}
u_t= u_i, \text{ w.p } \sigma_i
\end{align*}
for any $u_i\in\mathds{U}$ with $\sigma_i>0$ for all $i$. 

We note that for the convergence of the learning algorithm, it is sufficient for the hidden state process to converge to its invariant distribution under the exploration policy. Hence, any policy that leads the hidden state process to its invariant measure and visits every action with positive probability can be used for the exploration. For example, the control action can also be chosen to be a function of the most recent measurement and randomized (as long as all actions have positive probability of being selected for every measurement realization); this would again lead to a uniquely ergodic hidden state process under our assumptions.

The algorithm differs from the usual Q-value iteration:
\begin{itemize}
\item[(i)] The distribution of $I_1^t$, which is the consecutive N-window information variable when we hit the $(I,u)$, is generally different for every $t$ and the pair $(I,u)$ is not a controlled Markov process.

In other words, the controlled transitions are time dependent, that is, if we assume $N=1$ then for some $I=(y_t,y_{t-1},u_{t-1})$ and $u=u_{t}$:
\begin{align*}
&Pr(I_1^t=(y'_{t+1},y'_t,u'_t)|z=(y_t,y_{t-1},u_{t-1}),u_t)=\mathds{1}_{\{y_t=y_t',u_t=u_t'\}}Pr(y_{t+1}|y_t,y_{t-1},u_t,u_{t-1})
\end{align*} 
is not stationary and might change at every time step $t$, since $Pr(y_{t+1}|y_t,y_{t-1},u_t,u_{t-1})$ depends on the marginal distribution of $x_{t-1}$ ($x_{t-N}$ in the general case).


\item[(ii)] Here, we only observe the cost realizations of the underlying state process $\{x_t\}_t$ and the control actions. For example, if we assume that $N=1$ then the cost we observe is $c(x_t,u_t)$. However, $c(x_t,u_t)$ depends on $(I,u)$ pair randomly and in a time dependent way so that for some $I=(y_t,y_{t-1},u_{t-1})$ and $u=u_{t}$:
\begin{align*}
C_t(I,u)=c(x_t,u_t) \in B ,\quad \text{w.p. } Pr(X_t \in \{x: c(x,u_t) \in B\} |y_t,y_{t-1},u_{t-1})
\end{align*}
where $Pr(dx_t|y_t,y_{t-1},u_{t-1})$ can be seen as some {\it pseudo-belief} on the underlying state variable given $I=(y_t,y_{t-1},u_{t-1})$, the most recent $N=1$ information variables. In other words, $Pr(dx_t|y_t,y_{t-1},u_{t-1})$ is the Bayesian update of $\pi_{t-1}$, the marginal distribution of the true state $x_{t-1}$ at the time step $t-1$, using $I=(y_t,y_{t-1},u_{t-1})$ and thus, it is time dependent. \hfill $\diamond$
\end{itemize}


We will observe that, if one assumes that the hidden state process, $\{x_t\}_t$ is positive Harris recurrent, or at least, admits a unique invariant probability measure $\pi^*$ under a stationary exploration policy $\gamma$, then the average of approximate state transitions gets closer to
\begin{align}\label{inv_kernel}
&P^{*}(I_{t+1}|I_t,u_t):=\hat{\eta}^N((\pi^*,I_{t+1})|(\pi^*,I_t),u_t)
\end{align}
with $\hat{\eta}^N$ is defined as in  (\ref{eta_N}) and (\ref{eta_N_1}). In particular, if we assume $N=1$, then we write
\begin{align}\label{inv_kernel_1}
&P^{*}(I_{t+1}=(y_{t+1}',y_t',u_t')|I_{t}=(y_t,y_{t-1},u_{t-1}),u_t)=\mathds{1}_{\{y_t'=y_t,u_t'=u_t\}}P^{\pi^*}(y_{t+1}|y_t,y_{t-1},u_t,u_{t-1})
\end{align}
where $P^{\pi^*}(y_{t+1}|y_t,y_{t-1},u_t,u_{t-1})$ denotes the distribution of $y_{t+1}$ when the marginal distribution on $x_{t-1}$ is given by the invariant measure $\pi^*$.

We also have that the sample path averages of the random cost realizations get close to,
\begin{align*}
&C^*(I,u)=\hat{c}(\pi^*,I,u)=\int_\mathds{X}c(x,u)P^{\pi^*}(dx|I)
\end{align*}
where,  $P^*(x|I)$ is the Bayesian update of $\pi^*$, using $I$ and $\hat{c}(\pi^*,I,u)$ is defined as in (\ref{hat_cost}). If we assume $N=1$, we can write for some $I=(y_1,y_0,u_0)$ and $u=u_1$
\begin{align}\label{C_star}
&C^*(y_1,y_0,u_0,u_1)=\hat{c}(\pi^*,(y_1,y_0,u_0),u_1)=\int_\mathds{X}c(x_1,u_1)P^{\pi^*}(dx_1|y_1,y_0,u_0).
\end{align}

Now consider the following fixed point equation
\begin{align}\label{fixed}
Q^*(I,u)=C^*(I,u)+\beta\sum_{I'} P^*(I'|I,u)\min_vQ^*(I',v)
\end{align}
where $P^*$ is defined in (\ref{inv_kernel}) and $C^*$ is defined in (\ref{C_star}). 

The existence of a such fixed point follows from usual contraction arguments. The same fixed equation can also be written as, for $N=1$, and for $I=(y_1,y_0,u_0)$ and $u=u_1$
\begin{align}\label{fixed_y}
Q^*\left((y_1,y_0,u_0),u_1\right)=C^*\left((y_1,y_0,u_0),u_1\right)+\beta\sum_{y_2\in\mathds{Y}}P^{\pi^*}(y_2|y_1,y_0,u_1,u_0)\min_{v\in\mathds{U}}Q^*\left((y_2,y_1,u_1),v\right).
\end{align}

For the rest of the paper, we will use the following notation
\begin{align}\label{valueUpDefn}
V^*(I)&:=\min_{v\in\mathds{U}}Q^*(I,v)\\
V_t(I)&:=\min_{v\in\mathds{U}}Q_t(I,v).
\end{align}

We note that the stationary distribution $\pi^*$ does not have to be calculated by the decision maker. The Q value iterations given in (\ref{q_alg}) only use the finite-memory variables $I$, and $\pi^*$ is not used in the iterations. We will show that the algorithm naturally converges to (\ref{fixed}), if the state process is positive Harris recurrent, or at least, admits a unique invariant probability measure $\pi^*$ under a stationary exploration policy $\gamma$, where $\pi^*$ will be the stationary distribution of the hidden state process $x_t$ under the exploration policy. The performance loss will depend on the stationary distribution $\pi^*$ that is learned via the exploration policy, however, we will establish further upper bounds that are uniform over such $\pi^*$ which decrease exponentially with the window size $N$ (see Theorem \ref{curtis_thm}, equation (\ref{curtis_bound}), Corollary \ref{cor1} and Corollary \ref{cor2}).

\begin{assumption}\label{partial_q}
\hfill
\begin{itemize}
\item [1.] $\alpha_t(I,u)=0$ unless $(I_t,u_t)=(I,u)$. Furthermore,
\[\alpha_t(I,u) = {1 \over 1+ \sum_{k=0}^{t} 1_{\{I_k=I, u_k=u\}} }\]
We note that, this means that $\alpha_{k}(I,u)=\frac{1}{k}$ if $I_{k}=I,u_{k}=u$, if $k$ is the instant of the $k$th visit to $(I,u)$, as this will be crucial in the averaging of the Markov chain dynamics (see Remark \ref{onLearningRates}).
\item [2.] Under every stationary \{memoryless or finite memory exploration\} policy, say $\gamma$, the true state process, $\{X_t\}_t$, is positive Harris recurrent and in particular admits a unique invariant measure $\pi_\gamma^*$.
\item [3.] During the exploration phase, every $(I,u)$ pair is visited infinitely often.
\end{itemize}
\end{assumption}




\begin{theorem}\label{main_thm}
Under Assumption \ref{partial_q},
\begin{itemize}
\item[i.] The algorithm given in (\ref{q_alg}) converges almost surely to $Q^*$ which satisfies (\ref{fixed}).
\item[ii.]  For any policy $\gamma^N$ that satisfies $Q^*(I,\gamma^N(I))=\min_uQ^*(I,u)$, if we assume that the controller starts using $\gamma^N$ at time $t=N$ (after observing at least $N$ information variables), then denoting the prior distribution of $X_N$ by $\pi_N^-$, conditioned on the first $N$ step information variables we have
\begin{align*}
E\left[J_\beta(\pi_N^-,\mathcal{T},\gamma^N)-J_\beta^*(\pi_N^-,\mathcal{T})|I_0^N\right]\leq \frac{2\|c\|_\infty }{(1-\beta)}\sum_{t=0}^\infty\beta^t L_t
\end{align*}
where $L_t$ is defined in (\ref{loss_constant}) such that
\begin{align*}
L_t:=\sup_{\hat{\gamma}\in\hat{\Gamma}}E_{\pi_0^-}^{\hat{\gamma}}\left[\|P^{\pi_t^-}(X_{t+N}\in\cdot|Y_{[t,t+N]},U_{[t,t+N-1]})-P^{\pi^*}(X_{t+N}\in\cdot|Y_{[t,t+N]},U_{[t,t+N-1]})\|_{TV}\right]
\end{align*} 
 and $\pi^*$ is the invariant measure on $x_t$ under the exploration policy $\gamma$.
\end{itemize}
\end{theorem}

\begin{proof}
For the proof of $i$, that is for the convergence of Q-learning, we separate the iterations into sub-iterations which are linear (as in \cite{jaakkola1994convergence}, where this superposition principle of linear systems theory is utilized in showing the convergence of standard Q-learning algorithm). For the first part of the separated iterations, we use the fact that the dynamic programming equation is a contraction to prove its convergence which is similar to the traditional Q learning algorithms. For the remaining part of the iteration, we analyze the asymptotic behaviour of $I_1^t$, in which we distinguish our analysis from the traditional Q learning algorithms: For the usual Q iterations, one needs to study $X_1$ that is the consecutive state following some $(x,u)$ pair, and we have that $X_1\sim \mathcal{T}(\cdot|x,u)$. Thus, it is distributed independently and identically given $(x,u)$ which allows one to use Robbins–Monro type algorithms, to show the convergence. However, distributions of $I_1^t$'s are time dependent and not controlled Markovian. To study the asymptotic behavior of $I_1^t$, we construct a different pair process which is Markov and we use ergodicity properties of Markov chains.

We first prove that the process $Q_t$, determined by the algorithm in (\ref{q_alg}), converges almost surely to $Q^*$. We  define 
\begin{align*}
\Delta_t(I,u)&:=Q_t(I,u)-Q^*(I,u)\\
F_t(I,u)&:=C_t(I,u)+\beta V_t(I^t_1)-Q^*(I,u)\\
\hat{F}_t(I,u)&:=C^*(I,u)+\beta\sum_{I_1} V_t(I_1)P^*(I_1|I,u) -Q^*(I,u),
\end{align*}
where ($V_t$ is defined in \ref{valueUpDefn}).

Then, we can write the following iteration
\begin{align*}
\Delta_{t+1}(I,u)=(1-\alpha_t(I,u))\Delta_t(I,u)+\alpha_t(I,u) F_t(I,u).
\end{align*}
Now, we write $\Delta_t=\delta_t+w_t$ such that 
\begin{align*}
\delta_{t+1}(I,u)&=(1-\alpha_t(I,u))\delta_t(I,u)+\alpha_t(I,u) \hat{F}_t(I,u)\\
w_{t+1}(I,u)&=(1-\alpha_t(I,u))w_t(I,u)+\alpha_t(I,u) r_t(I,u)
\end{align*}
where $r_t:=F_t-\hat{F}_t=\beta V_t(I_1^t)-\beta \sum_{I_1}V_t(I_1)P^*(I_1|I,u) + C_t(I,u)- C^*(I,u)$. Next, we define
\begin{align*}
r_t^*(I,u)=\beta V^*(I^t_1)-\beta \sum_{I_1}V^*(I_1)P^*(I_1|I,u) + C_t(I,u)-C^*(I,u)
\end{align*}
We further separate $w_t=u_t+v_t$ such that
\begin{align*}
u_{t+1}(I,u)&=(1-\alpha_t(I,u))u_t(I,u)+\alpha_t(I,u) e_t(I,u)\\
v_{t+1}(I,u)&=(1-\alpha_t(I,u))v_t(I,u)+\alpha_t(I,u) r^*_t(I,u)
\end{align*}
where $e_t=r_t-r^*_t$. 



In the appendix (see Section \ref{v_kto0}), we show that $v_t(I,u)\to 0$ almost surely for all $(I,u)$. 

Now, we go back to the iterations:
\begin{align*}
\delta_{t+1}(I,u)&=(1-\alpha_t(I,u))\delta_t(I,u)+\alpha_t(I,u) \hat{F}_t(I,u)\\
u_{t+1}(I,u)&=(1-\alpha_t(I,u))u_t(I,u)+\alpha_t(I,u) e_t(I,u)\\
v_{t+1}(I,u)&=(1-\alpha_t(I,u))v_t(I,u)+\alpha_t(I,u) r^*_t(I,u).
\end{align*}
Note that, we want to show $\Delta_t=\delta_t+u_t+v_t \to 0$ almost surely and we have that $v_t(I,u)\to 0$ almost surely for all $(I,u)$. The following analysis holds for any path that belongs to the probability one event in which $v_t(I,u)\to 0$. For any such path and for any given $\epsilon>0$, we can find an $N<\infty$ such that $\|v_t\|_\infty<\epsilon$ for all $t>N$ as $(I,u)$ takes values from a finite set.

We now focus on the term $\delta_t + u_t$ for $t>N$:
\begin{align}\label{sum_proc}
(\delta_{t+1}+u_{t+1})(I,u)&=(1-\alpha_t(I,u))(\delta_t+u_t)(I,u)+\alpha_t(I,u) (\hat{F}_t+e_t)(I,u).
\end{align}
Observe that  for $t>N$,
\begin{align*}
(\hat{F}_t+e_t)(I,u)=(F_t-r_t^*)(I,u)=\beta V_t(I^t_1)-\beta V^*(I^t_1)&\leq \beta\max_{I,u}|Q_t(I,u)-Q^*(I,u)|=\beta\|\Delta_t\|_\infty\\
&\leq \beta \|\delta_t+u_t\|_\infty+\beta \epsilon
\end{align*}
where the last step follows from the fact that $v_t\to 0$ almost surely. By choosing $C<\infty$ such that $\hat{\beta}:=\beta(C+1)/C<1$, for $\|\delta_t+u_t\|_\infty>C\epsilon$, we can write that
\begin{align*}
 \beta \|\delta_t+u_t+\epsilon\|_\infty\leq \hat{\beta}\|\delta_t+u_t\|_\infty.
\end{align*}
Now we rewrite (\ref{sum_proc})
\begin{align}
(\delta_{t+1}+u_{t+1})(I,u)&=(1-\alpha_t(I,u))(\delta_t+u_t)(I,u)+\alpha_t(I,u) (\hat{F}_t+e_t)(I,u) \nonumber \\
&\leq (1-\alpha_t(I,u))(\delta_t+u_t)(I,u)+\alpha_t(I,u)  \hat{\beta}\|\delta_t+u_t\|_\infty \label{QBound12} \\
&<\|\delta_t+u_t\|_\infty \nonumber
\end{align}

Hence $\max_{I,u}((\delta_{t+1}+u_{t+1})(I,u))$ monotonically decreases for $\|\delta_t+u_t\|_\infty>C\epsilon$ and hence there are two possibilities: it either gets below $C\epsilon$ or it never gets below $C \epsilon$ in which case by the monotone non-decreasing property it will converge to some number, say $M_1$ with $M_1 \geq C\epsilon$. 

First, we show that once the process hits below $C\epsilon$, it always stays there. Suppose $\|\delta_t+u_t\|_\infty<C\epsilon$,
\begin{align}
(\delta_{t+1}+u_{t+1})(I,u)&\leq(1-\alpha_t(I,u))(\delta_t+u_t)(I,u)+\alpha_t(I,u) \beta \left(\|\delta_t+u_t\|_\infty+\epsilon\right) \nonumber \\
&\leq (1-\alpha_t(I,u)) C\epsilon + \alpha_t(I,u) \beta (C\epsilon + \epsilon) \nonumber \\
&= (1-\alpha_t(I,u)) C\epsilon  + \alpha_t(I,u) \beta (C+1) \epsilon  \nonumber\\
&\leq  (1-\alpha_t(I,u)) C\epsilon  + \alpha_t(I,u) C\epsilon, \quad (\beta(C+1)\leq C)  \nonumber \\
&=C\epsilon.  \nonumber
\end{align}

We now show that the latter, that is with the limit being $M_1 \geq C\epsilon$, is not possible. By (\ref{QBound12}), we have that for all $(I,u)$


\[\sum_t \alpha_t(I,u) \bigg( (\delta_t+u_t)(I,u)) - \hat{\beta}\|\delta_t+u_t\|_\infty \bigg) \leq (\delta_{0}+u_{0})(I,u) -\liminf_{N \to \infty} (\delta_t+u_t)(I,u)\]

If, pointwise, $(\delta_t+u_t)(I,u)$ admits a limit, the maximum over $(I,u)$ is attained by an individual $(I,u)$ beyond a finite index and thus we arrive at
\[\sum_t \alpha_t(I,u) \bigg( \|\delta_t+u_t\|_\infty - \hat{\beta}\|\delta_t+u_t\|_\infty \bigg) \leq (\delta_{0}+u_{0})(I,u) -\liminf_{t \to \infty} (\delta_t+u_t)(I,u)\]
which is a contradiction as $\alpha_t$ is not summable and the difference $\bigg( \|\delta_t+u_t\|_\infty - \hat{\beta}\|\delta_t+u_t\|_\infty \bigg)$, beyond a finite number, is bounded from below as $\hat{\beta} < 1$.

Thus, it suffices to show that $(\delta_t+u_t)(I,u)$ admits a limit. That this limit exists follows from \cite[Lemma 3]{jaakkola1994convergence} as (\ref{QBound12}) is an instance of bounded linear iteration under the assumption that $\|\delta_t+u_t\|_\infty$ is bounded.

This shows that the condition $\|\delta_t+u_t\|_\infty>C\epsilon$ cannot be sustained indefinitely for some fixed $C$ (independent of $\epsilon$). Hence, $(\delta_t+u_t)$ process converges to some value below $C\epsilon$ for any path that belongs to the probability one set. Then, we can write $\|\delta_t+u_t\|_\infty<C\epsilon$ for large enough $t$. Since $\epsilon > 0$ is arbitrary, taking $\epsilon \to 0$, we can conclude that $\Delta_t=\delta_t+u_t+v_t \to 0$ almost surely.

Therefore, the process $Q_t$, determined by the algorithm in (\ref{q_alg}), converges almost surely to $Q^*$. 

For item (ii), recall that 
\begin{align*}
Q^*(I,u)=C^*(I,u)+\beta\sum_{I_1} P^*(I_1|I,u)\min_vQ^*(I,v).
\end{align*}
This fixed point equation coincides with the DCOEs for the approximate belief MDP defined in (\ref{N_fixed}) and (\ref{N_1_fixed}). Hence, Using Theorem \ref{robust_bound}, any policy that satisfy $Q^*(I,\gamma^N(I))=\min_uQ^*(I,u)$  we can write
\begin{align*}
E\left[J_\beta(\pi_N^-,\mathcal{T},\gamma^N)-J_\beta^*(\pi_N^-,\mathcal{T})|I_0^N\right]\leq \frac{2\|c\|_\infty }{(1-\beta)}\sum_{t=0}^\infty\beta^t L_t
\end{align*}
 such that
\begin{align*}
L_t:=\sup_{\hat{\gamma}\in\hat{\Gamma}}E_{\pi_0^-}^{\hat{\gamma}}\left[\|P^{\pi_t^-}(X_{t+N}\in\cdot|Y_{[t,t+N]},U_{[t,t+N-1]})-P^{\pi^*}(X_{t+N}\in\cdot|Y_{[t,t+N]},U_{[t,t+N-1]})\|_{TV}\right]
\end{align*}
and $\pi^*$ is the invariant measure on $x_t$ under the exploration policy $\gamma$.
 \end{proof}

A few remarks about the result are now in order:
\begin{remark}\label{onLearningRates}
The learning rates for the standard Q-learning algorithm require:
\begin{itemize}
\item $\sum_k \alpha_k=\infty$
\item $\sum_k \alpha_k^2<\infty$.
\end{itemize}
In our case, we have a particular form. To justify this, we note that although these standard two assumptions on the learning rates may be sufficient for convergence of the algorithm, the limit fixed point equation (if one exists) will not necessarily be useful. Consider the following example where the state space is $\mathds{X}=\{-1,+1\}$ and transitions are deterministic such that $Pr(x_{t+1}=1|x_t=-1)=1$, $Pr(x_{t+1}=-1|x_t=+1)=1$ (leading to a periodic Markov chain). If one chooses the learning rates as $\alpha_{2k}=0$, $\alpha_{2k+1}=\sigma_k$ for every $k$ such that $\sigma_k$ is square summable but not summable, then even though the algorithm will converge, depending on the initial point, one of the transition models will always dominate the other. To avoid such examples, we choose the learning rates to be 'averaging' through time. \hfill $\diamond$
\end{remark}

\begin{remark}
 We caution the reader that our result assumes that the cost starts running after time $N$: that is the effective cost is: 
\begin{align}\label{termi_cost}
E\left[\sum_{k=N}^{\infty}\beta^{k-N}c(x_k,u_k)\right].
\end{align}
Of course, this criterion is also applicable if the system starts running prior to time $-N$ and the costs become in effect after time $0$.

If this criterion is not applicable, and the first $N$ stages are also crucial, (i) if $\beta$ is large enough, we can conclude that the first $N$ stages are not as critical for the analysis as their contributions will be minor in comparison with the future stages for the criterion, which can also be seen by considering this equivalent criterion to (\ref{criterion1}) and noting that for large enough $\beta$, the contributions of the first $N$ time stages become negligible:
\[(1 - \beta)E\left[\sum_{k=0}^{\infty}\beta^{k}c(x_k,u_k)\right].\]

(ii) On the other hand, if $\beta$ is not large and if the cost starts running at time 0, then, we can first run the Q-learning algorithm above to find the best $N$-window policies which optimizes (\ref{termi_cost}). The remaining question would be to optimize:
\begin{align}\label{finiteDPN}
E\left[\sum_{k=0}^{N-1}c(x_k,u_k)+V(I_k)\right]
\end{align}
as a finite-horizon optimal control problem with a terminal cost and the terminal cost $V$ can be estimated by (\ref{termi_cost}) via Theorem \ref{cont_bound} and Theorem \ref{main_thm}. The question then becomes how to select the first $N$ actions, leading to a problem with a finite search complexity for a finite horizon problem, without knowing the system dynamics. For this, one can run a MCMC algorithm in parallel simulations to find the optimal policy for the first $N$ time stages. Since the resulting policy minimizing (\ref{finiteDPN}) will be at least as good as the first $N$-window policy under the optimal (belief-MDP) policy (which is not designed to optimize (\ref{finiteDPN}) but the original cost (\ref{criterion1})), the bounds presented in Theorem \ref{robust_bound} will be applicable even when the cost criterion includes the first $N$ time stages.
\end{remark}

\begin{remark}
For the convergence rate of the Q-learning algorithm (\ref{q_alg}), it is clear that increasing the window size will increase the size of the state space, which in turn will slow down the convergence speed of the iterations. The sample complexity in Q-learning has been studied extensively for various learning rates. For linear learning rates, $\alpha_k=\frac{1}{k}$, as we use in this paper, the conclusion is that (see e.g. \cite{even2003learning}) the sample complexity is in the order of $\frac{(|\mathds{Y}|\times |\mathds{U}|)^{\frac{N}{1-\beta}}}{(1-\beta)^4\epsilon^2}$  when one seeks $\epsilon$ optimality in expectation. Thus, the number of samples needed to get an $\epsilon$-near estimate, in expectation increases exponentially in the window size. However, we note that this is the number of samples needed to get $\epsilon$-near to the limit Q value for the specific window size $N$, which gets closer to the optimal Q value of the original POMDP, exponentially fast with increasing $N$ under suitable filter stability assumptions (see Corollary \ref{cor1} and \ref{cor2}). We also note that for different system parameters, it has been shown that better convergence speeds can be achieved with carefully chosen learning rates (see e.g. \cite{wainwright2019stochastic, li2020sample}). Hence, a learning rate that is adaptive to the window size $N$ and target total approximation error (in near-optimality as well as the error due to sample complexity), can be used for a faster learning. 
\end{remark}

\section{On the Filter Stability and Convergence to Near Optimality under Filter Stability}\label{secfilterStable}
In this section, we discuss the (uniform filter stability) term ${\it L_t}$ defined in (\ref{loss_constant})
\begin{align*}
L_t:=\sup_{\hat{\gamma}\in\hat{\Gamma}}E_{\pi_0^-}^{\hat{\gamma}}\left[\|P^{\pi_t^-}(X_{t+N}\in\cdot|Y_{[t,t+N]},U_{[t,t+N-1]})-P^{\hat{\pi}}(X_{t+N}\in\cdot|Y_{[t,t+N]},U_{[t,t+N-1]})\|_{TV}\right]
\end{align*}

Before we introduce related definitions and notation, we again emphasize that for the performance of the approximate model introduce in Section \ref{app_section}, $L_t$ plays a crucial role. As we have noted earlier, $L_t$ is a term related to controlled filter stability, a general problem in which one is interested in how fast a process forgets its incorrect initial prior with increasing observation and control variables over time. In particular, any quantitative result for controlled filter stability bounding the errors given by $L_t$ can be used to study the performance of the approximation and the learning algorithm introduced in this paper. 

We state this formally as follows:

\begin{corollary}
Suppose the following assumptions hold:
\begin{itemize}
\item Assumption \ref{partial_q} holds.
\item Under the exploring policy, $\gamma$, the state process $\{x_t\}_t$ is irreducible.
\item The POMDP is such that the filter is stable uniformly over priors in expectation under total variation, that is $L_t \to 0$ as $N\to \infty$. 
\end{itemize}
Then, for any policy $\gamma^N$ that satisfies $Q^*(I,\gamma^N(I))=\min_uQ^*(I,u)$, if we assume that the controller starts using $\gamma^N$ at time $t=N$ (after observing at least $N$ information variables), then denoting the prior distribution of $X_N$ by $\pi_N^-$, conditioned on the first $N$ step information variables we have
\begin{align*}
E\left[J_\beta(\pi_N^-,\mathcal{T},\gamma^N)-J_\beta^*(\pi_N^-,\mathcal{T})|I_0^N\right] \to 0
\end{align*}
as $N\to \infty$.
\end{corollary}
\begin{proof}
The result follows directly from Theorem \ref{main_thm} and the definition of $L_t$ (see \ref{loss_constant})
\end{proof}

Notably, directly related to $L_t$, recent results in the literature, in particular \cite[Theorem 3.6]{MYRobustControlledFS} and  \cite[Theorem 3.3]{mcdonald2020exponential}, have presented explicit and sufficient conditions on the controlled filter stability problem under the total variation metric in expectation.

We will first focus on \cite[Theorem 3.3]{mcdonald2020exponential} and recall the Dobrushin coefficient for Markov kernels, which will provide exponential convergence rates for $L_t$.

\begin{definition}
For a given prior measure $\mu$ on $X_0$ and a policy $\gamma$, the one step predictor process is defined as the sequence of conditional probability measures 
\begin{align*}
\pi_{t-}^{\mu,\gamma}(\cdot)&=P^{\mu,\gamma}(X_t\in \cdot|Y_{[0,t-1]},U_{[t-1]}=\gamma(Y_{[0,t-1]}))
\end{align*}
where $P^{\mu,\gamma}$ is the probability measure induced by the prior $\mu$ and the policy $\gamma$, when $\mu$ is the probability measure on $X_0$.
\end{definition}

\begin{definition}
The filter process is defined as the sequence of conditional probability measures 
\begin{align}\label{filter}
\pi_{t}^{\mu,\gamma}(\cdot)&=P^{\mu,\gamma}(X_t\in \cdot|Y_{[0,t]},U_{[t-1]}=\gamma(Y_{[0,t-1]}))
\end{align}
where $P^{\mu,\gamma}$ is the probability measure induced by the prior $\mu$ and the policy $\gamma$.
\end{definition}



%

\begin{definition}\cite[Equation 1.16]{dobrushin1956central}
For a kernel operator $K:S_{1} \to \mathcal{P}(S_{2})$ (that is a regular conditional probability from $S_1$ to $S_2$) for standard Borel spaces $S_1, S_2$, we define the Dobrushin coefficient as:
\begin{align}
\delta(K)&=\inf\sum_{i=1}^{n}\min(K(x,A_{i}),K(y,A_{i}))\label{Dob_def}
\end{align}
where the infimum is over all $x,y \in S_{1}$ and all partitions $\{A_{i}\}_{i=1}^{n}$ of $S_{2}$.
\end{definition}
We note that this definition holds for continuous or finite/countable spaces $S_{1}$ and $S_{2}$ and $0\leq \delta(K)\leq 1$ for any kernel operator. 
\begin{example}\label{doub_exmp}
Assume for a finite setup, we have the following stochastic transition matrix
\begin{equation*}
K = 
\begin{pmatrix}
\frac{1}{3} & \frac{1}{3}  & \frac{1}{3} \\
0 & \frac{1}{2} & \frac{1}{2} \\
\frac{3}{4} & 0 & \frac{1}{4}
\end{pmatrix}
\end{equation*}
The Dobrushin coefficient is the minimum over any two rows where we sum the minimum elements among those rows. For this example, the first and the second rows give $\frac{2}{3}$, the first and the third rows give $\frac{7}{12}$ and the second and the third rows give $\frac{1}{4}$. Then the  Dobrushin coefficient is $\frac{1}{4}$.
\end{example}

Let
\[ \tilde{\delta}({\cal T}):=\inf_{u \in \mathds{U}} \delta({\cal T}(\cdot|\cdot,u)). \]
%


\begin{theorem} \cite[Theorem 3.3]{mcdonald2020exponential} \label{curtis_thm}
Assume that for $\mu,\nu \in \P(\mathds{X})$, we have $\mu\ll\nu$. Then we have
\begin{align*}
E^{\mu,\gamma}\left[\|\pi_{n+1}^{\mu,\gamma}-\pi_{n+1}^{\nu,\gamma}\|_{TV}\right]\leq (1-\tilde{\delta}(\mathcal{T}))(2-\delta(Q))E^{\mu,\gamma}\left[\|\pi_{n}^{\mu,\gamma}-\pi_{n}^{\nu,\gamma}\|_{TV}\right].
\end{align*}
In particular, defining $\alpha:=(1-\tilde{\delta}(\mathcal{T}))(2-\delta(Q))$, we have
\begin{align*}
E^{\mu,\gamma}\left[\|\pi_{n}^{\mu,\gamma}-\pi_{n}^{\nu,\gamma}\|_{TV}\right]\leq 2\alpha^n.
\end{align*}
\end{theorem}


The absolute continuity assumption, that is $\mu\ll\nu$, can be interpreted as follows: assume that the true starting distribtion is $\mu$ but we start the update from an incorrect prior $\nu$. The error can be fixed with the information, $y_{[0,t]},u_{[0,t-1]}$ eventually, as long as, the incorrect starting distribution $\nu$, puts on a positive measure to every event that the real starting distribution $\mu$ puts on a positive measure. However, if it is not the case, that is, if the incorrect starting distribution $\nu$ puts $0$ measure to some event, that $\mu$, puts positive measure to, information variables are not sufficient to fix the starting error occurring from that $0$ measure event. Of course this would not be feasible as the prior would not be compatible with the measured data. In any case, in our setup, the incorrect prior serves as an approximation and this ca be made to satisfy the absolute continuity condition by design: this will be the invariant measure on the state under the exploration policy.

Recall that the Q learning iteration that uses finite window information variables, learns the Q values for approximate states of the form $(\pi^*, I_t^N)$, instead of the true states $(\pi_{t-N}^-,I_t^N)$. Theorem \ref{curtis_thm} suggests that the approximation error arising from using the stationary distribution, $\pi^*$, instead of $\pi_{t-N}^-$, can be fixed with the information variables $I_t^N$, if $\pi^*$ captures the non zero events of $\pi_{t-N}^-$, that is if $\pi_{t-N}^-\ll\pi^*$.

In particular, since $\tilde{\delta}(\mathcal{T})$ is a uniform Dobrushin coefficient over all control actions, the above bound is valid under any control action process. Thus, if $\pi_{t}^-\ll\pi^*$ for all $t$, then we can write
\begin{align}\label{curtis_bound}
L_t&=\sup_{\hat{\gamma}\in\hat{\Gamma}}E_{\pi_0^-}^{\hat{\gamma}}\left[\|P^{\pi_t^-}(X_{t+N}\in\cdot|Y_{[t,t+N]},U_{[t,t+N-1]})-P^{\hat{\pi}}(X_{t+N}\in\cdot|Y_{[t,t+N]},U_{[t,t+N-1]})\|_{TV}\right]\nonumber\\
&\leq 2\alpha^N
\end{align}
for all $t$.

\begin{corollary}[to Theorem \ref{main_thm} and \ref{curtis_thm}]\label{cor1}
Assume the following holds:
\begin{itemize}
\item Assumption \ref{partial_q} holds.
\item The state space, $\mathds{X}$, is finite.
\item Under the exploring policy, $\gamma$, the state process $\{x_t\}_t$ is irreducible.
\item $\alpha:=(1-\tilde{\delta}(\mathcal{T}))(2-\delta(O))<1$.
\end{itemize}
Then, for any policy $\gamma^N$ that satisfies $Q^*(I,\gamma^N(I))=\min_uQ^*(I,u)$, if we assume that the controller starts using $\gamma^N$ at time $t=N$ (after observing at least $N$ information variables), then denoting the prior distribution of $X_N$ by $\pi_N^-$, conditioned on the first $N$ step information variables we have
\begin{align*}
E\left[J_\beta(\pi_N^-,\mathcal{T},\gamma^N)-J_\beta^*(\pi_N^-,\mathcal{T})|I_0^N\right]\leq \frac{4\|c\|_\infty }{(1-\beta)^2}\alpha^N.
\end{align*}
\end{corollary}
\begin{proof}{Proof.}
Note that, by Theorem \ref{main_thm},
\begin{align*}
E\left[J_\beta(\pi_N^-,\mathcal{T},\gamma^N)-J_\beta^*(\pi_N^-,\mathcal{T})|I_0^N\right]\leq \frac{2\|c\|_\infty }{(1-\beta)}\sum_{t=0}^\infty\beta^t L_t
\end{align*}

If the state process $x_t$ is irreducible under the exploring policy, then by Kac's Lemma (\cite{kac}), we have that 
\begin{align*}
\pi^*(x)>0, \quad \forall x\in\mathds{X}.
\end{align*}
Hence, using the inequality  (\ref{curtis_bound}), we complete the proof.
\end{proof}

\begin{corollary}[to Theorem \ref{main_thm} and \ref{curtis_thm}]\label{cor2}
Assume the following holds:
\begin{itemize}
\item Assumption \ref{partial_q} holds.
\item $\mathds{X}\subset \mathds{R}^m$ for some $m<\infty$.
\item The transition kernel $\mathcal{T}(\cdot|x_0,u_0)$ admits a density function $f$ with respect to a measure $\phi$ such that $\mathcal{T}(dx_1|x_0,u_0)=f(x_1,x_0,u_0)\phi(dx_1)$ and $f(x_1,x_0,u_0)>0$ for all $x_1,x_0,u_0$.
\item $\alpha:=(1-\tilde{\delta}(\mathcal{T}))(2-\delta(O))<1$.
\end{itemize}
Then, for any policy $\gamma^N$ that satisfies $Q^*(I,\gamma^N(I))=\min_uQ^*(I,u)$, if we assume that the controller starts using $\gamma^N$ at time $t=N$ (after observing at least $N$ information variables), then denoting the prior distribution of $X_N$ by $\pi_N^-$, conditioned on the first $N$ step information variables we have
\begin{align*}
E\left[J_\beta(\pi_N^-,\mathcal{T},\gamma^N)-J_\beta^*(\pi_N^-,\mathcal{T})|I_0^N\right]\leq \frac{4\|c\|_\infty }{(1-\beta)^2}\alpha^N.
\end{align*}
\end{corollary}
\begin{proof}{Proof.}
Note that, by assumption  $\mathcal{T}(dx_1|x_0,u_0)=f(x_1,x_0,u_0)\phi(dx_1)$ and $f(x_1,x_0,u_0)>0$ for all $x_1,x_0,u_0$ and hence, under the exploration policy $\gamma$, the state process $x_t$ is $\phi$-irreducible and admits a unique invariant measure, say $\pi^*$. Using the assumptions, we can also write that for any $A\in\B(\mathds{X})$ with $\phi(A)>0$
\begin{align*}
\pi^*(A)=\int_\mathds{Z}\int_A \int_{\mathds{U}}f(x_1,x_0,u_0)\gamma(du_0)\phi(dx_1)\pi^*(dx_0)>0
\end{align*}
which implies that $\phi\ll\pi^*$. Note that the transition kernel $\mathcal{T}(\cdot|x,u)$ is absolutely continuous with respect to $\phi$ for every $(x,u)$, and thus, for the predictor $\pi_t^-$ at any time step $t$, we can write that $\pi_t^-\ll\phi\ll\pi^*$.  

Hence, inequality (\ref{curtis_bound}) and Theorem \ref{main_thm} concludes the proof.
\end{proof}


\section{Numerical Study}\label{num_study}

In this section, we present a numerical study for the proven results.

The example we use is a machine repair problem. In this model, we have $\mathds{X,Y,U}=\{0,1\}$ with
\begin{align*}
x_t=&\begin{cases}1 \quad \text{ machine is working at time t }\\
0  \quad \text{ machine is not working at time t }.\end{cases}
u_t=&\begin{cases}1 \quad \text{ machine is being repaired at time t }\\
0  \quad \text{ machine is not being repaired at time t }.\end{cases}
\end{align*}
The probability that the repair was successful given initially the machine was not working is given by $\kappa$:
\begin{align*}
Pr(x_{t+1}=1|x_t=0,u_t=1)=\kappa
\end{align*}
The probability that the machine breaks down while in a working state is given by $\theta$:
\begin{align*}
Pr(x_t=0|x_t=1,u_t=0)=\theta
\end{align*}
The probability that the channel gives an incorrect measurement is given by $\epsilon$:
\begin{align*}
Pr(y_t=1|x_t=0)=Pr(y_t=0|x_t=1)=\epsilon
\end{align*}
The one stage cost function is given by
\begin{align*}
c(x,u)=&\begin{cases}R+E \quad  &x=0,u=1 \\
E  \quad  &x=0, u=0 \\
0 \quad &x=1,u=0\\
R \quad &x=1, u=1\end{cases}
\end{align*}
where $R$ is the cost of repair and 
$E$ is the cost incurred by a broken machine.

We study the example with discount factor $\beta=0.8$,  and present three different results by changing the other parameters. 

{\bf First example.} For the first case, we take $\epsilon=0.3$, $\kappa=0.8$, $\theta=0.1$ and, $R=5, E=1$. For the exploring policy, we use a random policy such that $Pr(\gamma(x)=0)=\frac{1}{2}$ and $Pr(\gamma(x)=1)=\frac{1}{2}$ for all $x$. Under this policy, $x_t$ admits a stationary policy $\pi^*(\cdot)=0.1\delta_{0}(\cdot)+0.9\delta_{1}(\cdot)$.

We have proved in Theorem \ref{main_thm} that the Q iteration given by (\ref{q_alg}) converges to the Q-values of the approximate belief-MDP defined in (\ref{N_fixed}). Defining
\begin{align*}
V_t(I)&:=\min_{v\in\mathds{U}}Q_t(I,v),
\end{align*}
in the next graphs, we plot $\sup_{I}|V_t(I)-J_\beta^N(\pi^*,I)|$ for $N=0,1,2$:
\begin{center}
\includegraphics[scale=0.5]{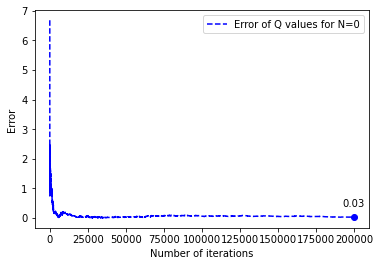}\includegraphics[scale=0.5]{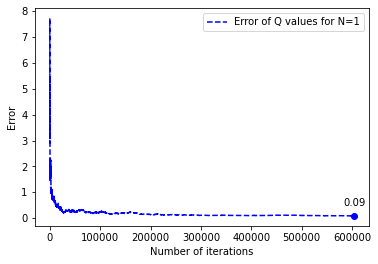}
\includegraphics[scale=0.5]{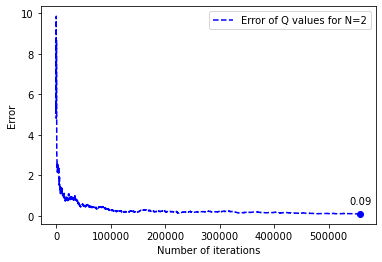}
\end{center}
We now show the performance of $\gamma^N$ that are found using the Q-values for different values of $N$. Recall that in Theorem \ref{main_thm}, we have showed that  
$$E\left[J_\beta(\pi_N^-,\mathcal{T},\gamma^N)-J_\beta^*(\pi_N^-,\mathcal{T})|I_0^N\right]\leq \frac{2\|c\|_\infty }{(1-\beta)}\sum_{t=0}^\infty\beta^t L_t$$ where
\begin{align*}
L_t:=\sup_{\hat{\gamma}\in\hat{\Gamma}}E_{\pi_0^-}^{\hat{\gamma}}\left[\|P^{\pi_t^-}(X_{t+N}\in\cdot|Y_{[t,t+N]},U_{[t,t+N-1]})-P^{\pi^*}(X_{t+N}\in\cdot|Y_{[t,t+N]},U_{[t,t+N-1]})\|_{TV}\right]
\end{align*} 

For the examples, we use the following upper bound for $L_t$
\begin{align*}
L:=\sup_{\pi\in\P(\mathds{X})}\sup_{\hat{\gamma}\in\hat{\Gamma}}E_{\pi}^{\hat{\gamma}}\left[\|P^{\pi}(X_{N}\in\cdot|Y_{[0,N]},U_{[0,N-1]})-P^{\pi^*}(X_{N}\in\cdot|Y_{[0,N]},U_{[0,N-1]})\|_{TV}\right]
\end{align*}
such that 
$$E\left[J_\beta(\pi_N^-,\mathcal{T},\gamma^N)-J_\beta^*(\pi_N^-,\mathcal{T})|I_0^N\right]\leq \frac{2\|c\|_\infty }{(1-\beta)^2}L$$

In the following, to estimate $J_\beta^*(\mu,\mathcal{T})$, we simply use the smallest value of $J_\beta(\mu,\mathcal{T},\gamma^N)$ among the different $N$ values. Furthermore, we scale the $L$ values to show the rate dependence between $J_\beta(\mu,\mathcal{T},\gamma^N)-J_\beta^*(\mu,\mathcal{T})$ and $L$ more clearly:
\begin{center}
\includegraphics[scale=0.6]{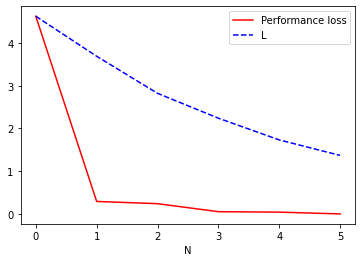}
\end{center}
It is clearly seen that the decrease rate for $L$ dominates the decrease rate for the error.

{\bf Second Example.} For the second case, we take $\epsilon=0.1$, $\kappa=0.9$, $\theta=0.3$  and, $R=5, E=1$. For the exploring policy, we again use a random policy such that $Pr(\gamma(x)=0)=\frac{1}{2}$ and $Pr(\gamma(x)=1)=\frac{1}{2}$ for all $x$. Under this policy, $x_t$ admits a stationary policy $\pi^*(\cdot)=0.29\delta_{0}(\cdot)+0.71\delta_{1}(\cdot)$.

The following shows the error between $V_t(I)=\min_vQ_t(I,v)$ and $J_\beta^N(\pi^*,I)$ for $N=0,1,2$:
\begin{center}
\includegraphics[scale=0.5]{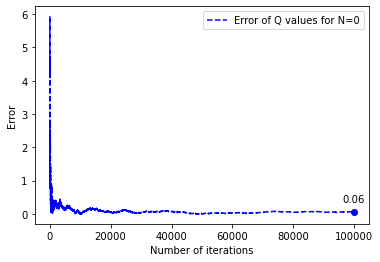}\includegraphics[scale=0.5]{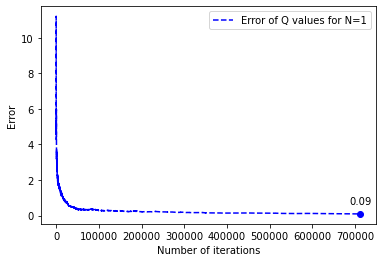}
\includegraphics[scale=0.5]{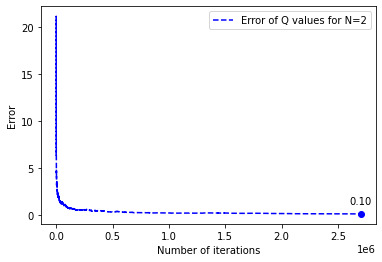}
\end{center}

The next graph shows $J_\beta(\mu,\mathcal{T},\gamma^N)-J_\beta^*(\mu,\mathcal{T})$ and scaled $L$:
\begin{center}
\includegraphics[scale=0.6]{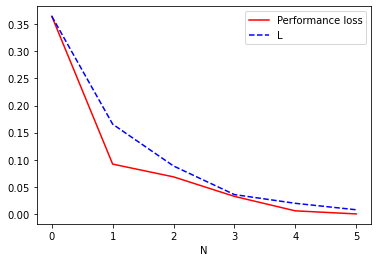}
\end{center}

{\bf Third Example: } Note that for the previous examples, we had $\alpha:=(1-\tilde{\delta}(\mathcal{T}))(2-\delta(O))>1$, however, the error still decreases since $\alpha>1$ condition is only a sufficient condition and that the error still converges to $0$, even when $\alpha>1$ in some cases. For the last example, we set parameters so that $\alpha<1$. The parameters are chosen as follows:
\begin{align*}
Pr(x_1=0|x_0=0,u_0=0)&=0.9, \quad Pr(x_1=0|x_0=0,u_0=1)=0.6\\
Pr(x_1=0|x_0=1,u_0=0)&=0.4,\quad Pr(x_1=0|x_0=1,u_0=1)=0.1
\end{align*}
Notice that, we manipulated some of the parameters, to make the $\alpha$ coefficient suitable for the purpose of the example.
For the measurement channel:
\begin{align*}
Pr(y=0|x=0)=0.7,\quad Pr(y=1|x=1)=0.7.
\end{align*}
For the cost function, we choose $R=3$, and $E=1$.

 We again use a random policy such that $Pr(\gamma(x)=0)=\frac{1}{2}$ and $Pr(\gamma(x)=1)=\frac{1}{2}$ for all $x$. Under this policy, $x_t$ admits a stationary policy $\pi^*(\cdot)=0.42\delta_{0}(\cdot)+0.58\delta_{1}(\cdot)$.

The convergence of the Q-values can be seen as:
\begin{center}
\includegraphics[scale=0.5]{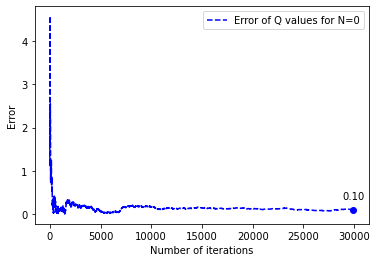}\includegraphics[scale=0.5]{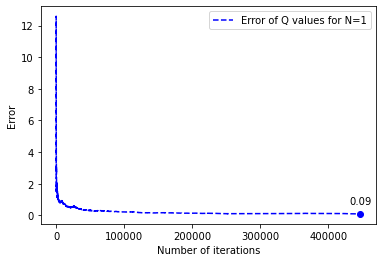}
\includegraphics[scale=0.5]{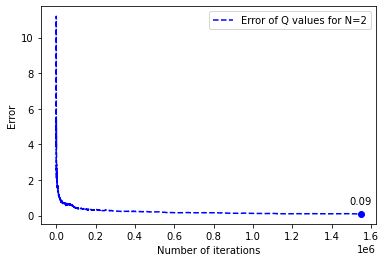}
\end{center}

This setup gives $\alpha=0.7$. The following graph shows the error $J_\beta(\mu,\mathcal{T},\gamma^N)-J_\beta^*(\mu,\mathcal{T})$, $L$ and $\alpha^N$ terms. We scale all of them to make them start from the same point to emphasize the decrease rates.
\begin{center}
\includegraphics[scale=0.6]{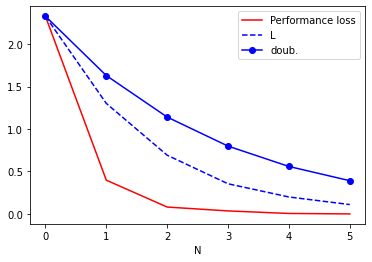}
\end{center}

\section{Concluding Remarks and a Discussion}
We studied the convergence of an approximate Q learning algorithm for partially observed stochastic control systems that uses finite window history variables. We provided sufficient conditions that guarantee the algorithm to converge, and then we provided the approximate belief-MDP model that the limit fixed equation corresponds to. Furthermore, we provided bounds for the approximate policy that is learned with the proposed algorithm in comparison to the true optimal policy that could be designed if the system and channel were known apriori. In particular, we obtained explicit error bounds between the resulting policy's performance and the optimal performance as a function of the memory length and a coefficient related to filter stability.

The setup we used for this paper focused on continuous state space and finite observation and action spaces. An immediate future direction is for continuous observation and action spaces, in which case, continuity properties of the transition model $\mathcal{T}(\cdot|x,u)$ on $u$ and continuity properties of the channel $O(dy|x)$ on $x$ is crucial and sufficient for consistent discretization of the observation and action spaces leading to analogous stability results. One condition, for example, would be that the channel be of the form $\int_A O(dy|x)=\int_A f(x,y)dy$ for all Borel $A$ with $f$ continuous in both variables. 

It is also our goal to generalize such results to multi-agent problems where finite history policies will likely lead to new insights towards tractable solutions in both stochastic team theory and game theory. 

\newpage

\appendix
\section{Proof of $v_t(I,u)\to 0$}\label{v_kto0}
\hfill

We will show that $v_t(I,u)\to 0$ almost surely for all $(I,u)$. We prove the claim only for $N=1$ case for simplicity and let $I=(y_1,y_0,u_0)$ and $u=u_1$ for some $(y_1,y_0,u_0,u_1)\in\mathds{Y}^2\times\mathds{U}^2$. The proof for general $N$ follows from essentially same steps. We have
\begin{align*}
v_{t+1}(I,u)&=(1-\alpha_t(I,u))v_t(I,u)+\alpha_t(I,u) r^*_t(I,u).
\end{align*}
When the learning rates are chosen such that 
$\alpha_t(I,u)=0$ unless $(I_t,U_t)=(I,u)$, and,
\[\alpha_t(I,u) = {1 \over 1+ \sum_{k=0}^{t} 1_{\{I_k=I, U_k=u\}} }\]
this term reduces to
\begin{align*}
v_{t+1}(I,u)=\frac{\sum_{k=0}^{t-1} r^*_{k}(I,u) \mathds{1}_{\{I_k,U_k=I,u\}}}{\sum_{k=0}^{t-1} \mathds{1}_{\{I_k,U_k=I,u\}}}.
\end{align*}
Recall that
\begin{align*}
r_k^*(I,u)=\beta V^*(I^k_1)-\beta \sum_{I_1}V^*(I_1)P^*(I_1|I,u) + C_k(I,u)-C^*(I,u).
\end{align*}
Hence, we will first analyze the asymptotic behavior of 
\begin{align*}
&I_1^k:=(Y_{k+1},Y_{k},U_{k}),\\
&I_k,U_k=(Y_k,Y_{k-1},U_k,U_{k-1}).
\end{align*}
 To analyze the asymptotic behavior of these variables, we will make use of the Markov chain theory. We will show that 
\begin{align*}
I_1^k,I_k,U_k,X_{k-1}=\{Y_{k+1},Y_k,Y_{k-1},U_k,U_{k-1},X_{k-1}\}
\end{align*}
form a Markov chain under the exploration policy $\gamma$. Then, we will use stationary distribution of this Markov chain to analyze the asymptotic behavior of $(I_1^k,I_k,U_k)$. We write
\begin{align*}
&Pr(Y_{k+1},Y_{k},Y_{k-1},U_{k},U_{k-1},X_{k-1}|y_k,y_{k-1},\dots,y_0,u_{k-1},\dots,u_0,x_{k-2},\dots,x_0)\\
&=\mathds{1}_{\{Y_k,Y_{k-1},U_{k-1}=y_k,y_{k-1},u_{k-1}\}}Pr(Y_{k+1}|y_k,x_{k-1},u_k,u_{k-1})Pr(X_{k-1}|y_{k-1},x_{k-2},u_{k-2})\gamma(U_k)\\
&=Pr(Y_{k+1},Y_k,U_k,X_{k-1},U_{k-2}|y_k,y_{k-1},x_{k-2},u_{k-1},u_{k-2})
\end{align*}
where we use $\gamma$ as the probability measure of the exploring policy. Above, we used that 
\begin{align*}
&Pr(Y_{k+1}|y_k,\dots,y_0,x_{k-1},\dots,x_0,\dots,u_{k},\dots,u_0)=Pr(Y_{k+1}|y_k,x_{k-1},u_k,u_{k-1})\\
&Pr(X_{k-1}|y_{k-1},\dots,y_0,x_{k-2},\dots,x_0,u_{k-2}\dots,u_0)=Pr(X_{k-1}|y_{k-1},x_{k-2},u_{k-2}).
\end{align*}
Hence, the joint process $(Y_{k+1},Y_k,Y_{k-1},U_k,U_{k-1},X_{k-1})$ is a Markov chain. One can show that it has a unique invariant measure under the assumption that the state process $X_k$ admits a unique invariant measure under the exploration policy $\gamma$. 

We denote the stationary distribution of $X_k$ by $\pi^*$, and denote the probability measure induced on the joint process and marginals of the process by this stationary distribution by $P^{\pi^*}$ with an abuse of notation. Then, for any measurable function $f$
\begin{align*}
&\lim_{t\to\infty}\frac{1}{t}\sum_{k=0}^{t-1}f(I_1^{k},I_k,U_k,X_{k})=\int f(y'_2,y'_1,y'_0,u'_1,u'_0,x'_0)P^{\pi^*}(dx'_0,y'_0,y'_1,y'_2,u'_0,u'_1).
\end{align*}

In particular, we have that
\begin{align*}
&\lim_{t\to\infty}\frac{\frac{1}{t}\sum_{k=0}^{t-1} V^*_{k}(I,u) \mathds{1}_{\{I_k,u_k=I,u\}}}{\frac{1}{t}\sum_{k=0}^{t-1} \mathds{1}_{\{I_k,u_k=I,u\}}}=\frac{\int V^*(y'_2,y'_1,u'_1)\mathds{1}_{\{y'_1,y'_0,u'_1,u'_0=y_1,y_0,u_1,u_0\}}P^{\pi^*}(y'_0,y'_1,y'_2,u'_0,u'_1)}{\int \mathds{1}_{\{y'_1,y'_0,u'_1,u'_0=y_1,y_0,u_1,u_0\}}P^{\pi^*}(y'_0,y'_1,y'_2,u'_0,u'_1)}\\
&=\frac{\int_{y'_2} V^*(y'_2,y_1,u_1)P^{\pi^*}(y'_2,y'_1=y_1,y'_0=y_0,u'_1=u_1,u'_0=u_0)}{P^{\pi^*}(y'_1=y_1,y'_0=y_0,u'_1=u_1,u'_0=u_0)}\\
&=\sum_{\mathds{Y}}V^*(y'_2,y_1,u_1)P^{\pi^*}(y'_2|y_1,y_0,u_1,u_0) \\
&=\sum_{I_1}V^*(I_1)P^{*}(I_1|I,u)
\end{align*}
where $P^{\pi^*}(y_2|y_1,y_0,u_1,u_0)$ is the distribution of $y_2$ when the $x_0$'s marginal distribution is given by $\pi^*$ and $$P^{*}(I_1=(y_2,y_1',u_1')|I,u):=\mathds{1}_{\{y_1'=y_1,u_1'=u_1\}}P^{\pi^*}(y_2|y_1,y_0,u_1,u_0)$$
as defined in (\ref{inv_kernel}) and (\ref{inv_kernel_1}).

Using similar arguments, one can also show that for $I=(y_1,y_0,u_0)$ and $u=u_1$
\begin{align*}
&\lim_{k\to\infty}\frac{1}{k}\sum_{k'=0}^{k-1}C_{k'}(I,u)=\int_{\mathds{X}} c(x_1,u_1)P^{\pi^*}(dx_1|y_1,y_0,u_0)\\
&=C^*(y_1,y_0,u_0,u_1)=C^*(I,u).
\end{align*}

Thus, we have that
\begin{align*}
v_{t+1}(I,u)&=\frac{1}{t}\sum_{k=0}^{t-1} r^*_{k}(I,u)\to 0
\end{align*}
almost surely for all $(I,u)$.


\section{Proof of Theorem \ref{cont_bound}}\label{proof_cont_bound}

\begin{lemma}\label{y_bound}
We have that for any $\pi,\pi^*\in\P(\mathds{X})$ and for any $(y,u)_{[t,t-N]}:=\{y_t,\dots,y_{t-N},u_t,\dots,u_{[t-N]}\}\in\mathds{Y}^N\times\mathds{U}^N$
\begin{align*}
&\|P^{\pi}(Y_{t+1}\in\cdot|(y,u)_{[t,t-N]})-P^{\pi^*}(Y_{t+1}\in\cdot|(y,u)_{[t,t-N]})\|_{TV}\\
&\quad \leq \|P^{\pi}(X_t\in\cdot|y_{[t,t-N]},u_{[t-1,t-N]})-P^{\pi^*}(X_t\in\cdot|y_{[t,t-N]},u_{[t-1,t-N]})\|_{TV}
\end{align*}
\end{lemma}
\begin{proof}{Proof.}
Let $f$ is a measurable function of $\mathds{Y}$ such that $\|f\|_\infty\leq1$. We write
\begin{align*}
&\int f(y_{t+1})P^{\pi}(dy_{t+1}|(y,u)_{[t,t-N]})-\int f(y_{t+1})P^{\pi^*}(dy_{t+1}|(y,u)_{[t,t-N]})\\
&=\int f(y_{t+1})O(dy_{t+1}|x_{t+1})\mathcal{T}(dx_{t+1}|x_t,u_t)P^{\pi}(dx_t|y_{[t,t-N]},u_{[t-1,t-N]})\\
&\qquad\qquad - \int f(y_{t+1})O(dy_{t+1}|x_{t+1})\mathcal{T}(dx_{t+1}|x_t,u_t)P^{\pi^*}(dx_t|y_{[t,t-N]},u_{[t-1,t-N]})\\
& \leq\|P^{\pi}(X_t\in\cdot|y_{[t,t-N]},u_{[t-1,t-N]})-P^{\pi^*}(X_t\in\cdot|y_{[t,t-N]},u_{[t-1,t-N]})\|_{TV}
\end{align*}
at the last step, we used the fact that $\int  f(y_{t+1})O(dy_{t+1}|x_{t+1})\mathcal{T}(dx_{t+1}|x_t,u_t)$ is bounded by $1$ as a function of $x_t$ for $\|f\|_\infty\leq 1$. Taking the supremum over all $\|f\|_\infty\leq 1$ concludes the proof.
\end{proof}

\begin{proof}{Proof of the main theorem:}
Let $\hat{z}_0=(\pi_0^-,y_1,y_0,u_0)$. Then we write
\begin{align*}
\tilde{J}^N_\beta(\hat{z}_1)=J_\beta^N(\pi^*,y_1,y_0,u_0)=\min_{u_1\in\mathds{U}}\left(\hat{c}(\pi^*,y_1,y_0,u_0,u_1)+\beta\sum_{y_2\in\mathds{Y}}J_\beta^N(\pi^*,y_2,y_1,u_1)P^{\pi^*}(y_2|y_1,y_0,u_1,u_0)\right).
\end{align*}

Furthermore,
\begin{align*}
J_\beta^*(\hat{z}_1)&=J_\beta^*(\pi_0^-,y_1,y_0,u_0)\nonumber\\
&=\min_{u_{1}\in\mathds{U}}\bigg(\hat{c}(\pi_0^-,y_1,y_0,u_0,u_1)+\beta \sum_{y_2\in\mathds{Y}} J_\beta^*(\pi_1^-(\pi_0^-,y_0,u_0),y_{2},y_1,u_{1})P^{\pi_0^-}(y_{2}|y_1,y_0,u_1,u_0)\bigg).
\end{align*}

Note that, for any $\pi\in\P(\mathds{X})$, we have 
\begin{align*}
\tilde{J}_\beta^N(\pi,y_2,y_1,u_1)=\tilde{J}_\beta^N(\pi^*,y_2,y_1,u_1)=J_\beta^N(\pi^*,y_2,y_1,u_1).
\end{align*}
In particular, we have that
\begin{align*}
J_\beta^N(\pi^*,y_2,y_1,u_1)=\tilde{J}_\beta^N(\pi_1^-(\pi_0^-,y_0,u_0),y_2,y_1,u_1).
\end{align*}
Hence, we can write the following
\begin{align*}
&|\tilde{J}^N_\beta(\hat{z}_0)-J^*_\beta(\hat{z}_0)|\leq \max_{u_1\in\mathds{U}}\left|\hat{c}(\pi^*,y_1,y_0,u_0,u_1)-\hat{c}(\pi_0^-,y_1,y_0,u_0,u_1)\right|\\
&+ \max_{u_1\in\mathds{U}}\beta \left| \sum_{y_2\in\mathds{Y}}J_\beta^N(\pi^*,y_2,y_1,u_1)P^{\pi^*}(y_2|y_1,y_0,u_1,u_0)-\sum_{y_2\in\mathds{Y}}J_\beta^N(\pi^*,y_2,y_1,u_1)P^{\pi_0^-}(y_2|y_1,y_0,u_1,u_0)\right|\\
&+\max_{u_1\in\mathds{U}}\beta \sum_{y_2\in\mathds{Y}}\left|\tilde{J}_\beta^N(\pi_1^-(\pi_0^-,y_0,u_0),y_2,y_1,u_1)- J_\beta^*(\pi_1^-(\pi_0^-,y_0,u_0),y_{2},y_1,u_{1})\right|P^{\pi_0^-}(y_{2}|y_1,y_0,u_1,u_0).
\end{align*}
Note that, by the definition of $\hat{c}$, we have 
\begin{align*}
\left|\hat{c}(\pi^*,y_1,y_0,u_0,u_1)-\hat{c}(\pi_0^-,y_1,y_0,u_0,u_1)\right|\leq\|c\|_\infty \|P^{\pi^*}(X_1\in\cdot|y_1,y_0,u_0)-P^{\pi_0^-}(X_1\in\cdot|y_1,y_0,u_0)\|_{TV}
\end{align*}
If we denote $\hat{z}_1=\left((\pi_1^-(\pi_0^-,y_0,u_0),y_2,y_1,u_1\right)$, using Lemma \ref{y_bound} we can write

\begin{align*}
&E_{\pi_0^-}^{\gamma}\left[|\tilde{J}^N_\beta(\hat{z}_0)-J^*_\beta(\hat{z}_0)|\right]\leq \|c\|_\infty E_{\pi_0^-}^{\gamma}\left[\|P^{\pi^*}(X_1\in\cdot|Y_1,Y_0,U_0)-P^{\pi_0^-}(X_1\in\cdot|Y_1,Y_0,U_0)\|_{TV}\right]\\
&+\max_{u_1\in\mathds{U}}\beta \|J_\beta^N\|_\infty E_{\pi_0^-}^{\gamma}\left[\|P^{\pi_0^-}(y_{2}|Y_1,Y_0,U_1,U_0)-P^{\pi^*}(y_{2}|Y_1,Y_0,U_1,U_0)\|_{TV}\right]\\
&+\max_{u_1\in\mathds{U}}\beta E_{\pi_0^-}^{\gamma}\left[ \sum_{y_2\in\mathds{Y}}\left|\tilde{J}_\beta^N(\hat{z}_1)-J_\beta^*(\hat{z}_1)\right|P^{\pi_0^-}(y_{2}|Y_1,Y_0,U_1,U_0)\right]\\
&\leq \left(\|c\|_\infty + \beta\|J_\beta^N\|_\infty\right) L_0 + \max_{u_1\in\mathds{U}}\beta E_{\pi_0^-}^{\gamma}\left[ \sum_{y_2\in\mathds{Y}}\left|\tilde{J}_\beta^N(\hat{z}_1)-J_\beta^*(\hat{z}_1)\right|P^{\pi_0^-}(y_{2}|Y_1,Y_0,u_1,U_0)\right]\\
&\leq \left(\|c\|_\infty + \beta\|J_\beta^N\|_\infty\right) L_0 + \sup_{\hat{\gamma}\in\hat{\Gamma}}\beta E_{\pi_0^-}^{\hat{\gamma}}\left[\left|\tilde{J}_\beta^N(\hat{z}_1)-J_\beta^*(\hat{z}_1)\right|\right]
\end{align*}
where
\begin{align*}
L_t:=\sup_{\hat{\gamma}\in\hat{\Gamma}}E_{\pi_0^-}^{\hat{\gamma}}\left[\|P^{\pi_t^-}(X_{t+N}\in\cdot|Y_{[t,t+N]},U_{[t,t+N-1]})-P^{\pi^*}(X_{t+N}\in\cdot|Y_{[t,t+N]},U_{[t,t+N-1]})\|_{TV}\right]
\end{align*}
Then, following the same steps for $E_{\pi_0^-}^{\hat{\gamma}}\left[\left|\tilde{J}_\beta^N(\hat{z}_1)-J_\beta^*(\hat{z}_1)\right|\right]$ and repeating the procedure , one can see that
\begin{align*}
&E_{\pi_0^-}^{\gamma}\left[|\tilde{J}^N_\beta(\hat{z}_0)-J^*_\beta(\hat{z}_0)|\right]\leq  \left(\|c\|_\infty + \beta\|J_\beta^N\|_\infty\right) \sum_{t=0}^\infty \beta^t L_t
\end{align*}
Note that $\|J_\beta^N\|_\infty\leq \frac{\|c\|_\infty}{1-\beta}$. Hence we can conclude
\begin{align*}
&E_{\pi_0^-}^{\gamma}\left[|\tilde{J}^N_\beta(\hat{z}_0)-J^*_\beta(\hat{z}_0)|\right] \leq \frac{\|c\|_\infty }{(1-\beta)} \sum_{t=0}^\infty \beta^t L_t.
\end{align*}

\end{proof}

\section{Proof of Theorem \ref{robust_bound}}\label{proof_robust_bound}
\begin{proof}{Proof.}
We let $\hat{z}_0=(\pi_0^-,y_1,y_0,u_0)$. We denote the minimum selector for the approximate MDP by
\begin{align*}
u_1^N:=\tilde{\phi}^N(\pi_0^-,y_1,y_0,u_0)=\phi^N(\pi^*,y_1,y_0,u_0)
\end{align*}
and write
\begin{align*}
J_\beta(\hat{z}_0,\tilde{\phi}^N)&=J_\beta(\pi_0^-,y_1,y_0,u_0,\tilde{\phi}^N)\nonumber\\
&=\hat{c}(\pi_0^-,y_1,y_0,u_0,u_1^N)+\beta \sum_{y_2\in\mathds{Y}} J_\beta(\pi_1^-(\pi_0^-,y_0,u_0),y_{2},y_1,u_1^N,\tilde{\phi}^N)P^{\pi_0^-}(y_{2}|y_1,y_0,u_1^N,u_0).
\end{align*}
Furthermore, we write the optimality equation for $\tilde{J}_\beta^N$ as follows
\begin{align*}
\tilde{J}_\beta^N(\hat{z}_0)=\hat{c}(\pi^*,y_1,y_0,u_0,u_1^N)+\beta  \sum_{y_2\in\mathds{Y}}\tilde{J}_\beta^N(\pi_1^-(\pi_0^-,y_0,u_0),y_{2},y_1,u_1^N)P^{\pi^*}(y_{2}|y_1,y_0,u_1^N,u_0).
\end{align*}
Hence, denoting $\hat{z}_1:=\left(\pi_1^-(\pi_0^-,y_0,u_0),y_{2},y_1,u_1^N\right)$ and using Lemma \ref{y_bound}, we can write that

\begin{align*}
&E_{\pi_0^-}^{\hat{\gamma}}\left[\left|J_\beta(\hat{z}_0,\tilde{\phi}^N)-\tilde{J}_\beta^N(\hat{z}_0)\right|\right]\leq \sup_{\hat{\gamma}\in\hat{\Gamma}}E_{\pi_0^-}^{\hat{\gamma}}\left[\left|\hat{c}(\pi_0^-,Y_1,Y_0,U_0,U_1)-\hat{c}(\pi^*,Y_1,Y_0,U_0,U_1)\right|\right]\\
&\qquad+\sup_{\hat{\gamma}\in\hat{\Gamma}}E_{\pi_0^-}^{\hat{\gamma}}\left[\beta \sum_{y_2\in\mathds{Y}} J_\beta(\hat{z}_1,\tilde{\phi}^N)P^{\pi_0^-}(y_{2}|Y_1,Y_0,U_1,U_0)- \beta  \sum_{y_2\in\mathds{Y}}\tilde{J}_\beta^N(\hat{z}_1)P^{\pi^*}(y_{2}|Y_1,Y_0,U_1,U_0)\right]\\
&\leq \|c\|_\infty \sup_{\hat{\gamma}\in\hat{\Gamma}}E_{\pi_0^-}^{\hat{\gamma}}\left[\|P^{\pi^*}(X_1\in\cdot|Y_1,Y_0,U_0)-P^{\pi_0^-}(X_1\in\cdot|Y_1,Y_0,U_0)\|_{TV}\right]\\
&\qquad + \beta \|\tilde{J}_\beta^N\|_\infty \sup_{\hat{\gamma}\in\hat{\Gamma}} E_{\pi_0^-}^{\hat{\gamma}}\left[\|P^{\pi_0^-}(y_{2}|Y_1,Y_0,U_1,U_0)-P^{\pi^*}(y_{2}|Y_1,Y_0,U_1,U_0)\|_{TV}\right]\\
&\qquad +\beta  \sup_{\hat{\gamma}\in\hat{\Gamma}}E_{\pi_0^-}^{\hat{\gamma}}\left[\left|J_\beta(\hat{z}_1,\tilde{\phi}^N)-\tilde{J}_\beta^N(\hat{z}_1)\right|\right]\\
&\leq \|c\|_\infty L_0 + \beta \|\tilde{J}_\beta^N\|_\infty L_0 + \beta  \sup_{\hat{\gamma}\in\hat{\Gamma}}E_{\pi_0^-}^{\hat{\gamma}}\left[\left|J_\beta(\hat{z}_1,\tilde{\phi}^N)-\tilde{J}_\beta^N(\hat{z}_1)\right|\right].
\end{align*}
Following the same steps for $E_{\pi_0^-}^{\hat{\gamma}}\left[\left|J_\beta(\hat{z}_1,\tilde{\phi}^N)-\tilde{J}_\beta^N(\hat{z}_1)\right|\right]$ and repeating the same procedure,  with $\|\tilde{J}_\beta^N\|_\infty\leq \frac{\|c\|_\infty}{1-\beta}$ one can conclude that
\begin{align}\label{mid_bound}
&E_{\pi_0^-}^{\hat{\gamma}}\left[\left|J_\beta(\hat{z}_0,\tilde{\phi}^N)-\tilde{J}_\beta^N(\hat{z}_0)\right|\right]\leq\frac{\|c\|_\infty }{(1-\beta)}\sum_{t=0}^\infty \beta^t L_t.
\end{align}

Now, we go back to the theorem statement to write
\begin{align*}
E_{\pi_0^-}^{\hat{\gamma}}\left[\left|J_\beta(\hat{z}_0,\tilde{\phi}^N) -J^*_\beta(\hat{z}_0)\right|\right]&\leq E_{\pi_0^-}^{\hat{\gamma}}\left[ \left|J_\beta(\hat{z},\tilde{\phi}^N)-\tilde{J}_\beta^N(\hat{z})\right|\right]+ E_{\pi_0^-}^{\hat{\gamma}}\left[\left|\tilde{J}_\beta^N(\hat{z})-J_\beta^*(\hat{z})\right|\right]\\
&\leq \frac{2\|c\|_\infty }{(1-\beta)}\sum_{t=0}^\infty \beta^t L_t.
\end{align*}
The last step follows from (\ref{mid_bound}) and Theorem \ref{cont_bound}. 

\end{proof}

\bibliographystyle{plain}

\bibliography{SerdarBibliography,AliBibliography,references_acc,SerdarBibliography_acc,references}

\begin{thebibliography}{10}

\bibitem{Bil99}
P.~Billingsley.
\newblock {\em Convergence of probability measures}.
\newblock New York: Wiley, 2nd edition, 1999.

\bibitem{dobrushin1956central}
R.L. Dobrushin.
\newblock Central limit theorem for nonstationary {M}arkov chains. i.
\newblock {\em Theory of Probability \& Its Applications}, 1(1):65--80, 1956.

\bibitem{even2003learning}
Eyal Even-Dar, Yishay Mansour, and Peter Bartlett.
\newblock Learning rates for q-learning.
\newblock {\em Journal of machine learning Research}, 5(1), 2003.

\bibitem{Feinberg2}
E.~A. Feinberg.
\newblock Controlled {M}arkov processes with arbitrary numerical criteria.
\newblock {\em Th. Probability and its Appl.}, pages 486--503, 1982.

\bibitem{FeKaZa12}
E.A. Feinberg, P.O. Kasyanov, and N.V. Zadioanchuk.
\newblock Average cost {M}arkov decision processes with weakly continuous
  transition probabilities.
\newblock {\em Math. Oper. Res.}, 37(4):591--607, Nov. 2012.

\bibitem{FeKaZg14}
E.A. Feinberg, P.O. Kasyanov, and M.Z. Zgurovsky.
\newblock Partially observable total-cost {M}arkov decision process with weakly
  continuous transition probabilities.
\newblock {\em Mathematics of Operations Research}, 41(2):656--681, 2016.

\bibitem{golowich2022planning}
N.~Golowich, A.~Moitra, and D.~Rohatgi.
\newblock Planning in observable {POMDP}s in quasipolynomial time.
\newblock {\em arXiv preprint arXiv:2201.04735}, 2022.

\bibitem{hansen2013solving}
E.~A. Hansen.
\newblock Solving pomdps by searching in policy space.
\newblock {\em arXiv preprint arXiv:1301.7380}, 2013.

\bibitem{Her89}
O.~Hern\'andez-Lerma.
\newblock {\em Adaptive {M}arkov Control Processes}.
\newblock Springer-Verlag, 1989.

\bibitem{HernandezLermaMCP}
O.~Hernandez-Lerma and J.~B. Lasserre.
\newblock {\em Discrete-Time {M}arkov Control Processes: Basic Optimality
  Criteria}.
\newblock Springer, 1996.

\bibitem{jaakkola1995reinforcement}
Tommi Jaakkola, Satinder~P. Singh, and Michael~I. Jordan.
\newblock Reinforcement learning algorithm for partially observable markov
  decision problems.
\newblock In {\em Advances in neural information processing systems}, pages
  345--352, 1995.

\bibitem{kac}
M.~Kac.
\newblock On the notion of recurrence in discrete stochastic processes.
\newblock {\em Bull. AMS}, 53:1002--1010, 1947.

\bibitem{KSYWeakFellerSysCont}
A.~D. Kara, N.~Saldi, and S.~Y\"uksel.
\newblock Weak feller property of non-linear filters.
\newblock {\em Systems \& Control Letters}, 134:104--512, 2019.

\bibitem{kara2020near}
Ali~Devran Kara and Serdar Yuksel.
\newblock Near optimality of finite memory feedback policies in partially
  observed markov decision processes.
\newblock {\em arXiv preprint arXiv:2010.07452}, 2020.

\bibitem{Kri16}
V.~Krishnamurthy.
\newblock {\em Partially observed {M}arkov decision processes: from filtering
  to controlled sensing}.
\newblock Cambridge University Press, 2016.

\bibitem{li2020sample}
Gen Li, Yuting Wei, Yuejie Chi, Yuantao Gu, and Yuxin Chen.
\newblock Sample complexity of asynchronous q-learning: Sharper analysis and
  variance reduction.
\newblock {\em arXiv preprint arXiv:2006.03041}, 2020.

\bibitem{lin1992memory}
Long-Ji Lin and Tom~M Mitchell.
\newblock {\em Memory approaches to reinforcement learning in non-Markovian
  domains}.
\newblock Citeseer, 1992.

\bibitem{Lov91-(b)}
W.S. Lovejoy.
\newblock A survey of algorithmic methods for partially observed {M}arkov
  decision processes.
\newblock {\em Annals of Operations Research}, 28:47--66, 1991.

\bibitem{mccallum1997reinforcement}
Andrew McCallum.
\newblock Reinforcement learning with selective perception and hidden state.
\newblock {\em Doctoral dissertation, Department of Computer Science,
  University of Rochester.}, 1997.

\bibitem{mcdonald2018stability}
C.~McDonald and S.~Y\"uksel.
\newblock Converse results on filter stability criteria and stochastic
  non-linear observability.
\newblock {\em arXiv:1812.01772}, 2018.

\bibitem{MYCDC2019observability}
C.~McDonald and S.~Y{\"u}ksel.
\newblock Observability and filter stability for partially observed markov
  processes.
\newblock In {\em 2019 IEEE 58th Conference on Decision and Control (CDC)},
  pages 1623--1628. IEEE, 2019.

\bibitem{mcdonald2020exponential}
C.~McDonald and S.~Y\"uksel.
\newblock Exponential filter stability via {D}obrushin's coefficient.
\newblock {\em Electronic Communications in Probability}, 25, 2020.

\bibitem{MYRobustControlledFS}
C.~McDonald and S.~Y\"uksel.
\newblock Robustness to incorrect priors and controlled filter stability in
  partially observed stochastic control.
\newblock {\em SIAM Journal on Control and Optimization, to appear (also
  arXiv:2110.03720)}, 2022.

\bibitem{Par67}
K.R. Parthasarathy.
\newblock {\em Probability Measures on Metric Spaces}.
\newblock AMS Bookstore, 1967.

\bibitem{pineau2006anytime}
J.~Pineau, G.~Gordon, and S.~Thrun.
\newblock Anytime point-based approximations for large pomdps.
\newblock {\em Journal of Artificial Intelligence Research}, 27:335--380, 2006.

\bibitem{porta2006point}
J.~M. Porta, N.~Vlassis, M.~T.~J. Spaan, and P.~Poupart.
\newblock Point-based value iteration for continuous pomdps.
\newblock {\em Journal of Machine Learning Research}, 7(Nov):2329--2367, 2006.

\bibitem{Rhe74}
D.~Rhenius.
\newblock Incomplete information in {M}arkovian decision models.
\newblock {\em Ann. Statist.}, 2:1327--1334, 1974.

\bibitem{SaYuLi15c}
N.~Saldi, S.~Y{\"u}ksel, and T.~Linder.
\newblock On the asymptotic optimality of finite approximations to markov
  decision processes with borel spaces.
\newblock {\em Mathematics of Operations Research}, 42(4):945--978, 2017.

\bibitem{SYLTAC2017POMDP}
N.~Saldi, S.~Y\"uksel, and T.~Linder.
\newblock Finite model approximations for partially observed markov decision
  processes with discounted cost.
\newblock {\em IEEE Transactions on Automatic Control}, 65, 2020.

\bibitem{singh1994learning}
Satinder~P. Singh, Tommi Jaakkola, and Michael~I. Jordan.
\newblock Learning without state-estimation in partially observable markovian
  decision processes.
\newblock {\em Machine Learning Proceedings 1994}, pages 284--292, 1994.

\bibitem{smith2012point}
T.~Smith and R.~Simmons.
\newblock Point-based pomdp algorithms: Improved analysis and implementation.
\newblock {\em arXiv preprint arXiv:1207.1412}, 2012.

\bibitem{Mahajan2019}
J.~{Subramanian} and A.~{Mahajan}.
\newblock Approximate information state for partially observed systems.
\newblock In {\em 2019 IEEE 58th Conference on Decision and Control (CDC)},
  pages 1629--1636, 2019.

\bibitem{jaakkola1994convergence}
Tommi T.~Jaakkola, M.~I. Jordan, and S.~P. Singh.
\newblock On the convergence of stochastic iterative dynamic programming
  algorithms.
\newblock {\em Neural computation}, 6(6):1185--1201, 1994.

\bibitem{TsitsiklisQLearning}
J.~N. Tsitsiklis.
\newblock Asynchronous stochastic approximation and q-learning.
\newblock {\em Machine Learning}, 16:185--202, 1994.

\bibitem{villani2008optimal}
C.~Villani.
\newblock {\em Optimal transport: old and new}.
\newblock Springer, 2008.

\bibitem{spaan2005perseus}
N.~Vlassis and M.~T.~J. Spaan.
\newblock Perseus: Randomized point-based value iteration for pomdps.
\newblock {\em Journal of artificial intelligence research}, 24:195--220, 2005.

\bibitem{wainwright2019stochastic}
Martin Wainwright.
\newblock Stochastic approximation with cone-contractive operators: Sharper
  $l_\infty$-bounds for $q$-learning.
\newblock {\em arXiv preprint arXiv:1905.06265}, 2019.

\bibitem{Whi91}
C.C. White.
\newblock A survey of solution techniques for the partially observed {M}arkov
  decision process.
\newblock {\em Annals of Operations Research}, 32:215--230, 1991.

\bibitem{white1994finite}
C.~C. White-III and W.~T. Scherer.
\newblock Finite-memory suboptimal design for partially observed markov
  decision processes.
\newblock {\em Operations Research}, 42(3):439--455, 1994.

\bibitem{yu2008near}
Huizhen Yu and Dimitri~P Bertsekas.
\newblock On near optimality of the set of finite-state controllers for average
  cost pomdp.
\newblock {\em Mathematics of Operations Research}, 33(1):1--11, 2008.

\bibitem{Yus76}
A.A. Yushkevich.
\newblock Reduction of a controlled {M}arkov model with incomplete data to a
  problem with complete information in the case of {B}orel state and control
  spaces.
\newblock {\em Theory Prob. Appl.}, 21:153--158, 1976.

\bibitem{zhou2008density}
E.~Zhou, M.~C. Fu, and S.~I. Marcus.
\newblock A density projection approach to dimension reduction for
  continuous-state {P}{O}{M}{D}{P}s.
\newblock In {\em Decision and Control, 2008. CDC 2008. 47th IEEE Conference
  on}, pages 5576--5581, 2008.

\bibitem{zhou2010solving}
E.~Zhou, M.~C. Fu, and S.~I. Marcus.
\newblock Solving continuous-state {P}{O}{M}{D}{P}s via density projection.
\newblock {\em IEEE Transactions on Automatic Control}, 55(5):1101 -- 1116,
  2010.

\bibitem{ZhHa01}
R.~Zhou and E.A. Hansen.
\newblock An improved grid-based approximation algorithm for {POMDP}s.
\newblock In {\em Int. J. Conf. Artificial Intelligence}, pages 707--714, Aug.
  2001.

\end{thebibliography}

\end{document}